\documentclass[12pt]{article}

\usepackage{microtype}
\usepackage[pdftex]{graphicx}
\usepackage{subfigure}
\usepackage{booktabs} 
\usepackage{hyperref}
\usepackage{fullpage}

\usepackage{amsmath,amssymb}
\usepackage{enumitem}
\usepackage{comment} 
\usepackage{hyperref}
\hypersetup{colorlinks,
	linkcolor=blue,
	citecolor=blue,
	urlcolor=magenta,
	linktocpage,
	plainpages=false}
\usepackage[round, sort]{natbib}
\usepackage{bm}

\usepackage{thmtools}
\usepackage{thm-restate}

\usepackage{times}

\usepackage{amssymb}
\usepackage{amsfonts}
\usepackage{amsthm}

\usepackage[utf8]{inputenc} 
\usepackage[T1]{fontenc}    
\usepackage{url}            
\usepackage{booktabs}       
\usepackage{amsfonts}       
\usepackage{nicefrac}       
\usepackage{microtype}      

  \usepackage{mathtools}

\usepackage{tikz}
\usepackage{pgfplots}
\usetikzlibrary{pgfplots.groupplots}

\usepackage{caption}

\newenvironment{itemize*}%
{\begin{itemize}[leftmargin=*,topsep=0pt]%
		\setlength{\itemsep}{0pt}%
		\setlength{\parskip}{0pt}}%
	{\end{itemize}}
\newenvironment{enumerate*}%
{\begin{enumerate}[leftmargin=*,topsep=0pt]%
		\setlength{\itemsep}{0pt}%
		\setlength{\parskip}{0pt}}%
	{\end{enumerate}}

\allowdisplaybreaks

\newtheorem{theorem}{Theorem}[section]

\newtheorem{problem}{Problem}
\newtheorem{claim}[theorem]{Claim}

\newtheorem{lemma}[theorem]{Lemma}
\newtheorem{corollary}[theorem]{Corollary}
\newtheorem{proposition}[theorem]{Proposition}
\newtheorem{definition}[theorem]{Definition}

\newtheorem{remark}{Remark}[section]

\newcommand{\vct}{\boldsymbol }
\newcommand{\diag}{\mathrm{diag}}

\def\R{\mathbb{R}}

\def\cB{\mathcal{B}}

\def\cF{\mathcal{F}}

\def\cN{\mathcal{N}}
\def\cO{\mathcal{O}}
\def\cP{\mathcal{P}}

\def\cY{\mathcal{Y}}

\def\approxcorrect{\cmark\kern-1.4ex\raisebox{.30ex}{$\xmark$}}

\newcommand{\idxn}[1][]{\ifthenelse{\equal{#1}{}}{\mathbb{INDQ}_n}{\mathbb{INDQ}_{#1}}}

\newcommand{\beq}{\begin{equation}}

\newcommand{\eeq}{\end{equation}}

\newcommand{\bt}{\bm{t}}

\newcommand{\bx}{{\bm{x}}}

\newcommand{\bw}{{\bm{w}}}

\newcommand{\ba}{{\bm{a}}}

\newcommand{\bv}{{\bm{v}}}
\newcommand{\bz}{{\bm{z}}}

\newcommand{\by}{{\bm{y}}}

\newcommand{\btheta}{{\bm{\theta}}}
\newcommand{\bzeta}{{\bm{\zeta}}}
\newcommand{\btau}{{\bm{\tau}}}
\newcommand{\bdelta}{{\bm{\delta}}}

\newcommand{\onebb}{{\mathbf{1}}}

\newcommand{\bbeta}{{\boldsymbol{\beta}}}

\newcommand{\bb}{\vct{b}}

\newcommand{\bs}{\bm{s}}

\newcommand{\E}{\operatorname{\mathbb{E}}}

\def \endprf{\hfill {\vrule height6pt width6pt depth0pt}\medskip}

\newcommand{\Trace}{{\mathrm{Tr}}}

\DeclareMathOperator*{\argmax}{arg\,max}
\DeclareMathOperator*{\argmin}{arg\,min}

\newcommand{\wei}[1]{\textcolor{purple}{[WH: #1]}}

\usepackage{wrapfig}
\newcommand{\tMM}{\text{MM}}
\newcommand{\tSS}{\text{SS}}
\newcommand{\tRR}{\text{RR}}
\newcommand{\tLS}{\text{LS}}
\usepackage{authblk}
\title{Near-Optimal Linear Regression under Distribution Shift}
\author[$\thanks{Princeton University. Email: \url{qilei@princeton.edu} }$]{\ Qi Lei}
\author[$\thanks{Princeton University. Email: \url{huwei@cs.princeton.edu}}$]{Wei Hu}
\author[$\thanks{Princeton University. Email: \url{jasonlee@princeton.edu}}$]{Jason D. Lee}
\affil[ ]{}

\date{\today}

\begin{document}

\maketitle

\begin{abstract}


Transfer learning is essential when sufficient data comes from the source domain, with scarce labeled data from the target domain. We develop estimators that achieve minimax linear risk for linear regression problems under distribution shift. Our algorithms cover different transfer learning settings including covariate shift and model shift. We also consider when data are generated from either linear or general nonlinear models. We show that linear minimax estimators are within an absolute constant of the minimax risk even among nonlinear estimators for various source/target distributions.
\end{abstract}

\section{Introduction} 

The success of machine learning crucially relies on the availability of labeled data. The data labeling process usually requires extensive human labor and can be very expensive and time-consuming, especially for large datasets like ImageNet~\citep{deng2009imagenet}.
On the other hand, models trained on one dataset, despite performing well on test data from the same distribution they are trained on, are often sensitive to \emph{distribution shifts}, i.e., they do not adapt well to related but different distributions.
Even small distributional shift can result in substantial performance degradation~\citep{recht2018cifar,lu2020harder}. 

Transfer learning has been an essential paradigm to tackle the challenges associated with insufficient labeled data~\citep{pan2009survey,weiss2016survey,long2017deep}. The main idea is to make use of a \emph{source domain} with plentiful labeled data (e.g., ImageNet) and to learn a model that performs well on the \emph{target domain} (e.g., medical images) where few or no labels are available.
Despite the lack of labeled data, we may still use unlabeled data from the target domain, which are usually much easier to obtain and can provide helpful marginal distribution information about the target domain.
Although this approach is integral to many applications, many fundamental questions are left open even in very basic settings.

In this work, we focus on the setting of \emph{linear regression under distribution shift} and ask the fundamental question of how to optimally learn a linear model for the target domain, using labeled data from a source domain and unlabeled data (and possibly limited unlabeled data) from the target domain. 
We design a two-stage meta-algorithm that addresses this problem in various settings, including covariate shift (i.e., when $p(\bx)$ changes) and model shift (i.e., when $p(y|\bx)$ changes). Following the meta-algorithm, we develop estimators that achieve \emph{near minimax risk} (up to universal constant factors) among all linear estimation rules under some standard data concentration properties.
Here linear estimators refer to all estimators that depend linearly on the label vector; these include almost all popular estimators known in linear regression, such as ridge regression and its variants.
When the second moment matrix of input variables in source and target domains commute, we prove that our estimators achieve near minimax risk among all possible estimators. We also provide a separation result demonstrating our algorithm can be better than ridge regression  by a multiplicative factor of $\tilde O(d^{-1/4})$.

A crucial insight from our results is that when covariate shift is present, we need to apply data-dependent regularization that adapts to changes in the input distribution. For linear regression, this is characterized by the input covariances of source and target tasks, estimated using unlabeled data.
Our experiments verify that our estimator has significant improvement over ridge regression and similar heuristics.

\subsection{Related work}
\textbf{Different types of distribution shift} are introduced in \citet{storkey2009training,quionero2009dataset}. 
Specifically, covariate shift occurs when the marginal distribution on $P(\bx)$ changes from source to target domain~\citep{heckman1979sample,shimodaira2000improving,huang2007correcting}. \citet{wang2014active,wang2015generalization} tackle model shift ($P(y|\bx)$) provided the change is smooth as a function of $\bx$. \citet{sun2011two} design a two-stage reweighting method based on both covariate shift and model shift. Other methods like the change of representation, adaptation through instance pruning are proposed in \citet{jiang2007instance}. In this work, we focus on the above two kinds of distribution shifts. Other distribution shift settings involving label/target shift ($P(y)$) and conditional shift ($ P(\bx|y)$) are beyond the scope of this paper. Some prior work also focuses on these settings (See reference therein~\citep{saerens2002adjusting,zhang2013domain,lipton2018detecting}). For instance, \citet{zhang2013domain} exploits the benefit of multi-layer adaptation by a location-scale transformation on $\bx$.

\textbf{Transfer learning/domain adaptation} are sub-fields within machine
learning to cope with distribution shift.
A variety of prior work falls into the following categories. 1) Importance-reweighting is mostly used in the covariate shift~\citep{shimodaira2000improving,huang2007correcting,cortes2010learning,cortes2014domain}; 2) One fruitful line of work focuses on
exploring robust/causal features or domain-invariant representations~\citep{wu2019domain} through invariant risk minimization~\citep{arjovsky2019invariant} and its variants~\cite{ahuja2020empirical,rosenfeld2021the,pmlr-v130-kamath21a}, distributional robust minimization~\citep{sagawa2019distributionally}, human annotation~\citep{srivastava2020robustness}, adversarial training~\citep{long2017deep,ganin2016domain}, or through variants of consistency regularization~\citep{pan2010domain,cortes2011domain,long2013transfer,baktashmotlagh2013unsupervised,gong2013connecting,zhang2013domain,cortes2019adaptation,wang2014flexible,sun2016deep,sun2016geometric,gretton2012kernel,li2018domain,cai2021theory}; 3) 
Several approaches seek gradual domain adaptation~\citep{gopalan2011domain,gong2012geodesic,glorot2011domain,kumar2020understanding,abnar2021gradual,Chen2021} through self-training or a gradual change in the training distribution.  

Some more theoretical work derived generalization upper and lower bound based on the source and target distributional discrepancy~\citep{david2010impossibility,hanneke2019value,ben2010theory,zhao2019learning,albuquerque2019generalizing,tachet2020domain}.







\textbf{Near minimax estimations} are introduced in \citet{donoho1994statistical} for linear regression problems with Gaussian noise. For a more general setting, \citet{juditsky2009nonparametric} estimate the linear functional using convex programming. \citet{blaker2000minimax} compares ridge regression with a minimax linear estimator using weighted squared error.  \citet{kalan2020minimax} considers a setting similar to this work of minimax estimator under distribution shift but focuses on computing the lower bound for the linear and one-hidden-layer neural network under distribution shift.


\section{Preliminary}
\label{sec:prelim} 
We formalize the setting considered in this paper for transfer learning under the distribution shift.

\paragraph{Notation and setup.} Let $p_S(\bx)$ and $ p_T(\bx)$ be the marginal distribution for $\bx$ in source and target domain. The associated second-moment matrices are $\Sigma_S:=\E_{p_S}[\bx\bx^\top]$, and $\Sigma_T:=\E_{p_S}[\bx\bx^\top]$. Labeled data $(\bx,y)$ satisfies $\E_{p_S}[y|\bx]=\E_{p_T}[y|\bx]=f^*(\bx)$ and $y=f^*(\bx)+z$ with Gaussian noise $z\sim  \cN(0,\sigma^2)$. We consider both linear ($f^*(\bx):=\E[y|\bx]=\bx^\top \bbeta^*$) and general nonlinear data generation model. When the optimal linear model changes from source to domain we add a subscript for distinction, i.e., $\bbeta_S^*$ and $\bbeta_T^*$. We use bold ($\bx$) symbols for vectors, lower case letter ($x$) for scalars and capital letter ($A$) for matrices.

We observe $n_S, n_T$ labeled samples from source and target domain, and $n_U$ unlabeled target samples. Labeled data is scarce in target domain: $n_S\gg n_T$ and $n_T$ can be $0$. Specifically, data is collected as 
$X_S=[\bx_1^\top |\bx_2^\top |\cdots |\bx_{n_S}^\top]^\top \in \R^{n_S\times d}$, with $\bx_i,i\in [n_S]$ drawn from $p_S$, noise $\bz = [z_1,z_2,\cdots z_{n_S}]^\top, z_i \sim \cN(0,\sigma^2)$. $\by_S = [y_1,y_2,\cdots, y_{n_S}]^\top \in \R^{n_S}$, 
$\by_T\in \R^{n_T}$ and $X_U\in \R^{n_U\times d}$ are similarly defined). Denote by $\hat \Sigma_S = X_S^\top X_S/n_S$ the empirical second-moment matrix. The positive part of a number is denoted by $(x)_+$. 

\paragraph{Minimax (linear) risk. }
In this work, we focus on designing linear estimators $\hat \bbeta:\R^n\rightarrow \R^d, \by_S \rightarrow A\by_S$\footnote{$A\in \R^{d\times n}$ may depend in an arbitrary way on $X_S, n_S$, or $\Sigma_T$. The estimator is linear in the observation $\by_S$.} for $\bbeta^*_T\in \cB$. Here $\bbeta^*_T$ is the optimal linear model in target domain ($:=\argmin_{\bbeta}\E_{\bx\sim p_T,z\sim \cN(0,\sigma^2)}[(f^*(\bx)+z-\bx^\top\bbeta)^2]$).  \footnote{We do not distinguish linear and affine regression since one could simply add another constant coordinate to $\bx$ to take into consideration the intercept part.}

Our estimator is evaluated by the excess risk on target domain, with the worst case $\bbeta^*_T$ in some set $\cB$: $L_{\cB}(\hat\bbeta) = \max_{\bbeta^*\in \cB} \E_{\by_S}\E_{\bx\sim p_T} \left(\bx^\top(\hat \bbeta(\by_S) -\bbeta^*_T)\right)^2$.  
Minimax linear risk and minimax risk among all estimators are respectively defined as:
\begin{align*}
R_L(\cB) \equiv & \min_{\hat\bbeta \text{ linear in } \by_S} L_{\cB}(\hat\bbeta); ~~ R_N(\cB) \equiv \min_{\hat\bbeta} L_{\cB}(\hat\bbeta).
\end{align*} 
The subscript ``N" or ``L" is a mnemonic for ``non-linear" or ``linear" estimators. $R_N$ is the optimal risk with no restriction placed on the class of estimators. $R_L$ only considers the linear function class for $\hat\bbeta$. 
Minimax linear estimator and minimax estimator are the estimators that respectively attain $R_L$ and $R_N$ within universal multiplicative constants. Normally we only consider $\cB=\{\bbeta | \|\bbeta\|_2\leq r \}$. 
When there is no ambiguity, we simplify $\hat\bbeta(\by_S)$ by $\hat\bbeta$.

\paragraph{Our meta-algorithm. }  
Our paper considers different settings with distribution shift. Our methods are unified under the following meta-algorithm:
\begin{itemize}
	\item Step 1: Construct an unbiased sufficient statistic $\hat\bbeta_{\tSS}$\footnote{With samples $\by_S$, a  statistic $t=T(\by_S)$ is sufficient for the underlying parameter $\bbeta^*$ if the conditional probability distribution of the data $\by_S$, given the statistic $t=T(\by_S)$, does not depend on the parameter $\bbeta^*$. } for the unknown parameter. 
	\item Step 2: Construct $\hat\bbeta_{\tMM}$, a linear function of the sufficient statistic $\hat\bbeta_{\tSS}$ that minimizes $L_{\cB}(\hat\bbeta_{\tMM})$.
\end{itemize}

For each setting, we will show that $\hat\bbeta_{\tMM}$ achieves linear minimax risk $R_L$. Furthermore, under some conditions, the minimax risk $R_N$ is uniformly lower bounded by a universal constant times $L_{\cB}(\hat\bbeta_{\tMM})$.

\paragraph{Outline. } In the sections below, we tackle the problem in several different settings. In Section \ref{sec:cov_shift}, we design algorithms with only covariate shift and linear data-generation models ($f^*$ is linear) for unsupervised domain adaptation ($n_T=0$) in Section \ref{sec:linear_model}, and supervised domain adaptation ($n_T>0$) in Section \ref{sec:linear_source_and_target_data}. Section \ref{sec:nonlinear_model} is about linear regression with approximation error ($n_T=0$ and $f^*(\bx)$ is a general nonlinear function). Finally we consider model shift for linear models ($\bbeta_S^*\neq \bbeta_T^*$) in Section \ref{sec:model_shift}.

\section{Covariate shift with linear models}
\label{sec:cov_shift}
In this section, we consider the setting with only covariate shift and $f^*$ is linear. That is, only $\Sigma_S$ (marginal distribution $p_S(\bx)$) changes to $\Sigma_T$ (marginal distribution $p_T(\bx)$), but $f^*=\E[y|\bx]=\bx^\top \bbeta^*$ (conditional distribution $p(y|\bx)$) is shared. 

\subsection{Unsupervised domain adaptation with linear models}
\label{sec:linear_model}
We observe $n_S$ samples from source domain: $ \by_S=X_S \bbeta^* + \bz,\bz\sim \cN(0,\sigma^2 I)$ and only some unlabeled samples $X_U$ from the target domain. Our goal is to find the minimax linear estimator $\hat \bbeta_{\tMM}(\by_S)=A\by_S$ with some linear mapping $A$ that attains $R_L(\cB)$ \footnote{For linear estimator $\hat\bbeta=A\by_S$, $\by_S$ is the only source of randomness and $A$ depends on $X_S,n_S$, which are considered fixed. }.  

Following our meta-algorithm, let $\hat \bbeta_{\tSS}=\frac{1}{n_S} \hat\Sigma_S^{-1} X_S^\top \by_S$\footnote{Throughout the paper $\hat\Sigma_S^{-1}$ could be replaced by pseudo-inverse and our algorithm also applies when $n<d$.} be an unbiased sufficient statistic for $\bbeta^*$: 
\begin{align}
\notag 
\hat \bbeta_{\tSS} = & \frac{1}{n_S} \hat\Sigma_S^{-1} X_S^\top \by_S =  \frac{1}{n_S} \hat \Sigma_S^{-1} X_S^\top X_S \bbeta^* + \frac{1}{n_S}  \hat \Sigma_S^{-1} X_S^\top \bz.\\
\label{eqn:sequence_model}
= & \bbeta^* + \frac{1}{n_S}  \hat \Sigma_S^{-1} X_S^\top \bz \sim  \cN\left(\bbeta^*, \frac{\sigma^2}{n_S}\hat \Sigma_S^{-1} \right) .
\end{align}
The fact that $\hat\bbeta_{\tSS}(\by_S)$ is a sufficient statistic for $\bbeta^*$ is proven in Claim \ref{claim:sufficient_statistic} for a more general case, using the Fisher-Neyman factorization theorem. 
We prove that the minimax linear estimator is of the form $\hat\bbeta_{\tMM}=C\hat\bbeta_{\tSS}$ and then design algorithms that calculate the optimal $C$. 
\begin{claim}
\label{claim:A=A_1X^T} 
The minimax linear estimator is of the form $\hat \bbeta_{\tMM} = C\hat\bbeta_{\tSS}$ for some $C\in \R^{d\times d}$.  
\end{claim}



\paragraph{Warm-up: commutative second-moment matrices.}
In order to derive the minimax linear estimator, we first consider the simple case when $\Sigma_T$ and $\hat\Sigma_S$ are simultaneously diagonalizable. We note that under this setting, minimax estimation under covariate shift reduces to the well-studied problem of finding a minimax linear estimator under weighted square loss (see e.g., \citep{blaker2000minimax}). One could apply Pinsker's Theorem \citep{johnstone2011gaussian} and get an estimator function and the minimax risk with a closed form:
\begin{theorem}[Linear Minimax Risk with Covariate Shift]
 \label{thm:linear_minimax_estimator}
Suppose the observations follow sequence model $ \by_S = X_S \bbeta^* + \bz, \bz\sim \cN(0,\sigma^2 I_n)$. If $\Sigma_T=U\diag(\bt)U^\top$ and $\hat\Sigma_S\equiv X_S^\top X_S/n_S = U\diag(\bs) U^\top$, then the minimax linear risk 
 \begin{align*}
 R_L(\cB) \equiv & \min_{\hat\bbeta=A\by_S} \max_{\bbeta^*\in \cB} \E\|\Sigma_{T}^{1/2}(\hat \bbeta -\bbeta^*)\|^2\\
  = & \sum_i \frac{\sigma^2}{n_S} \frac{t_i}{s_i} \left(1-\frac{\lambda}{\sqrt{t_i}}\right)_+,
 \end{align*} 
 where $\cB=\{\bbeta|\|\bbeta\|\leq r\}$, and $\lambda= \lambda(r)$ is determined by 
$ \frac{\sigma^2}{n_S}\sum_{i=1}^d  \frac{1}{s_i}(\sqrt{t_i}/\lambda-1)_+ = r^2.$ The linear minimax estimator is given by:
\begin{align}
\label{eqn:best_linear_estimator}
\hat \bbeta_{\tMM} =  &\Sigma_T^{-1/2} U (I-\diag(\lambda/\sqrt{\bt}) )_+ U^\top \Sigma_T^{1/2}\hat\bbeta_{\tSS},\\
\notag 
&  \text{where }\hat \bbeta_{\tSS} = \frac{1}{n_S} \hat\Sigma_S^{-1} X_S^\top \by_S. 
\end{align} 
\end{theorem} 
Since $r$ is unknown in practice, we could simply view either $r$ or directly $\lambda$ as the tuning parameter. We compare the functionality of $\lambda$ with that of ridge regression: $\hat\bbeta_{\tRR}^{\lambda}=\argmin_{\hat\bbeta} \E\frac{1}{2n}\|X_S\hat\bbeta - \by_S\|^2 + \frac{\lambda}{2}\|\hat\bbeta\|^2= (\hat\Sigma_S+\lambda I)^{-1}X_S^\top \by_S/n_S$. 
For both algorithms, $\lambda$ balances the bias and variance: $\lambda=0$ gives an unbiased estimator, and a big $\lambda$ gives a (near) zero estimator with no variance. The difference is, the minimax linear estimator shrinks some signal directions based on the value of $t_i$, since the risk in those directions is downweighted in the target loss. The estimator tends to sacrifice the directions of signal where $t_i$ is smaller. Ridge regression, however, respects the value of $s_i$. A natural counterpart is for ridge to also regularize based on $\bt$: let $\hat\bbeta^\lambda_{\tRR,T}=\argmin \frac{1}{n}\|\Sigma_T^{1/2}(\bbeta- \hat\Sigma_S^{-1}X_S^\top\by_S)\|^2+\lambda \|\bbeta\|^2 = (\Sigma_T+\lambda I)^{-1}\Sigma_T \hat\bbeta_{\tSS}$. We will compare their performances in the experimental section.

\paragraph{Non-commutative second-moment matrices.} 

For non-commutative second-moment shift, we follow the same procedure. Our estimator is achieved by optimizing over $C$: $\hat \bbeta_{\tMM} = C \hat\bbeta_{\tSS}$: 
\begin{align}
\notag 
& R_L(\cB)\equiv  \min_{\hat \bbeta=A\by_S} \max_{\bbeta^*\in\cB} \E\|\Sigma^{1/2}_T(\hat \bbeta - \bbeta^*)\|^2_2 \\
\notag 
= & \min_{\hat \bbeta = C\hat\bbeta_{\tSS}} \max_{\|\bbeta^*\|\leq r} \left\{\|\Sigma_T^{1/2}(C-I)\bbeta^* \|^2_2  + \frac{\sigma^2}{n_S} \Trace(\Sigma_T^{1/2}C\hat \Sigma_S^{-1}C^\top\Sigma_T^{1/2} ) \right\} \tag{Claim \ref{claim:A=A_1X^T}} \\
\label{eqn:best_linear_estimator_noncommute}
= & \min_{\tau, C} \left\{ r^2 \tau + \frac{\sigma^2}{n_S} \Trace(\Sigma_T^{1/2}C\hat \Sigma_S^{-1} C^\top\Sigma_T^{1/2}) \right\}, \\
\notag 
& \text{ s.t. }  (C-I)^\top\Sigma_T (C-I) \preceq \tau I.
\end{align}
Unlike the commutative case, this problem does not have a closed form solution, but is still computable:
\begin{proposition}
\label{prop:convex_program} 
Problem \eqref{eqn:best_linear_estimator_noncommute} is a convex program and computable in polynomial-time.	
\end{proposition}

We achieve near-optimal minimax risk among all estimators under some conditions:
\begin{theorem}[Near minimaxity of linear estimators]
	\label{thm:1.25minimax_risk}
 The best linear estimator from \eqref{eqn:best_linear_estimator} or \eqref{eqn:best_linear_estimator_noncommute}  achieves minimax linear risk:
	$L_{\cB}(\hat\bbeta_{\tMM}) = R_L(\cB)$. 
	When $\Sigma_S,\Sigma_T$ commute, or $\Sigma_T$ is rank 1, our algorithms are near-optimal among nonlinear rules: $L_{\cB}(\hat\bbeta_{\tMM})\leq 1.25 R_N(\cB).$
\end{theorem}
Note that $R_N\leq R_L$ by definition. Therefore 1) our estimator $\hat\bbeta_{\tMM}$ is near-optimal, and 2) our lower bound for $R_N$ is tight. Lower bounds (without matching upper bounds) for general non-commutative problem is presented in \citet{kalan2020minimax} and we improve their result for the commutative case and provide a matching algorithm. Their lower bound scales with $\frac{d}{n_S}\min_i \frac{t_i}{s_i}$ for large $r$, while ours becomes $\frac{1}{n_S}\sum_i \frac{t_i}{s_i}$. Our lower bound is always larger and thus tighter, and potentially arbitrarily larger when $\max_i \frac{t_i}{s_i}$ and $\min_i \frac{t_i}{s_i}$ are very different.
We defer our proof to the appendix. 

\subsection{Connection to ridge regression}
From a probabilistic perspective, ridge regression is equivalent to maximum a posteriori (MAP) inference with a Gaussian prior: $\bbeta^*\sim \cN(0,r^2I)$ (see e.g. \citet{murphy2012machine}). 
Similarly, instead of considering a worst-case risk that minimizes $L_{\cB}(\hat\bbeta) :=\max_{\bbeta^*\in \cB}\E_{\by_S}\|\Sigma_T^{1/2}(\hat\bbeta(\by_S) - \bbeta^*)\|^2 $, one could also study the average setting that minimizes $\bar L_{\cB}:= \E_{\bbeta^*\sim \cN(0,r^2 I)} \E_{\by_S}\|\Sigma_T^{1/2}(\hat\bbeta(\by_S) - \bbeta^*)\|^2$ instead. 
With distribution shift, the performance is evaluated on $\Sigma_T$ instead of $\Sigma_S$. Interestingly with Gaussian prior, this does not give us a different algorithm other than the original ridge regression.

\begin{proposition}
The optimal estimator under Gaussian prior $\bbeta^*\sim \cN(0,r^2 I)$ evaluated on $p_T$ is:
\begin{align*}
\hat \bbeta \leftarrow & \argmin_{\bbeta = A\by_S } \E_{\bbeta^*\sim \cN(0,r^2 I)} \E_{\by_S} \E_{\bx\sim p_T} \left(\bx^\top (\bbeta - \bbeta^*)\right)^2\\
= & \frac{1}{n_S}(\frac{\sigma^2}{r^2 n_S} I +\hat \Sigma_S)^{-1} X_S^\top \by_S \\
\equiv & \argmin_{\hat\bbeta} \E\frac{1}{2n}\|X_S\hat\bbeta - \by_S\|^2 + \frac{\lambda}{2}\|\hat\bbeta\|^2\\
= & (\hat\Sigma_S + \lambda I)^{-1}\hat \Sigma_S \hat \bbeta_{\tSS}=: \hat\bbeta_{\tRR}^{\lambda},
\end{align*}
when $\lambda = \sigma^2/(n_S r^2)$. Namely, the average-case best linear estimator with Gaussian prior is equivalent to ridge regression with regularization strength $\lambda=\frac{\sigma^2/n_S}{r^2}$: the variance ratio between the noise distribution and prior distribution. 
\end{proposition}
	
Even though ridge regression achieves the optimal risk in the average sense, it could be much worse than the minimax linear estimator in the worst case. We prove a separation result on a specific example (that is deferred to the appendix). 
\begin{remark}[Benefit of minimax linear estimator]
	\label{remark:order-wise_better}

	There is an example that $R_L(\cB)\leq \cO(d^{-1/4} L_{\cB}(\hat\bbeta_{\tRR}^\lambda))$ even with the optimal hyperparameter $\lambda$. \footnote{Note this goes without saying that our method can also be order-wise better than ordinary least square, which is a special case of ridge regression by setting $\lambda=0$.}	
\end{remark}
\paragraph{Adaptation on the prior distribution.} 
With specific problems, one should adjust the prior distribution instead of simply assume $\bbeta^*\sim \cN(0,r^2)$. If one replaces the prior by $\bbeta^*\sim \cN(\hat\bbeta_{\tSS},r^2)$, one could get another heuristic method:
\begin{proposition}
Let $\hat \bbeta_{\tSS}$ be the estimator from ordinary least square: $\hat\bbeta_{\tSS}= \hat \Sigma_S^{-1}X_S^\top\by_S/n_S$. The optimal estimator under Gaussian prior $\bbeta^*\sim \cN(\hat \bbeta_{\tSS},r^2 I)$ evaluated on $p_T$ is:
\begin{align*}
\hat \bbeta \leftarrow & \argmin_{\bbeta = A\by_S } \E_{\bbeta^*\sim \cN(\hat\bbeta_{\tSS},r^2 I)} \E_{\by_S} \E_{\bx\sim p_T} \left(\bx^\top (\bbeta - \bbeta^*)\right)^2\\
= & \frac{1}{n_S}(\frac{\sigma^2}{r^2 n_S} I + \Sigma_T)^{-1}\Sigma_T \hat\Sigma_S^{-1}  X_S^\top \by_S \\
\equiv &\argmin_{\bbeta} \|\Sigma_T^{1/2}(\bbeta- \hat\bbeta_{\tSS})\|^2+\lambda \|\bbeta\|^2\\
= & (\Sigma_T+\lambda I)^{-1}\Sigma_T \hat\bbeta_{\tSS}=:\hat\bbeta^\lambda_{\tRR,T},
\end{align*}
when $\lambda = \sigma^2 /(n_S r^2)$. 	
\end{proposition}

Comparing the closed-form estimator $\hat \bbeta_{\tRR,T}^{\lambda}:=(\Sigma_T+\lambda I)^{-1}\Sigma_T \hat\bbeta_{\tSS}$ to the original ridge regression $\hat \bbeta_{\tRR}^{\lambda}:=(\hat\Sigma_S+\lambda I)^{-1}\hat\Sigma_S \hat\bbeta_{\tSS}$, we could see that this algorithm regularizes $\hat\bbeta$ based on the signal strength from the target distribution, and it is equivalent to ridge regression by adjusting the prior distribution to center at $\hat\bbeta_{\tSS}$, the unbiased estimator for the ground truth $\bbeta^*$. We will compare both methods with our minimax estimator in the experimental section.

\subsection{Minimax linear estimator with finite unlabeled samples from target domain}

In practice, we have finite unlabeled samples $X_U\in \R^{n_U\times d},$ where we denote the empirical second-moment matrix as $\hat \Sigma_U = X_U^\top X_U/n_U$. Let $\hat L_{\cB}$ to denote the worst case excess risk measured on the observed target samples: $\hat L_\cB(\hat \bbeta) = \max_{\bbeta^*\in \cB} \E_{\by_S}\frac{1}{n_U} \|X_U(\hat\bbeta(\by_S)-\bbeta^*)\|^2  $. To find the best linear estimator that minimizes $\hat L_{\cB}$, our proposed algorithm becomes: 
\begin{align} 
\label{eqn:cov_shift_alg_finite_n_U}
\hat C \leftarrow & \min_{\tau, C} \left\{ r^2 \tau + \frac{\sigma^2}{n_S} \Trace(C\hat \Sigma_S^{-1} C^\top\hat\Sigma_U) \right\},\\  
\notag
& \text{ s.t. }  (C-I)^\top\hat \Sigma_U (C-I) \preceq \tau I.
\end{align} 
Let $\hat\bbeta = \hat C \hat\Sigma_S^{-1}X_S^\top \by_S/n_S$. We want to show that in spite of the existence of estimation error due to the replacement of $\Sigma_T$ with $\hat\Sigma_T$, our generated $\hat\bbeta$ still achieves minimax linear risk (up to constant multiplicative error).


For simplicity, in this section we assume input samples are centered: $\E_{p_S}[\bx]=\E_{p_T}[\bx]=0$. This assumption results in no loss of generality. Since the sample mean is more sample-efficient to estimate than covariance matrix, one will be able to first estimate the mean and center the data. We assume some standard light-tail property on the target samples: 
\begin{definition}[$\rho^2$-subgaussian distribution]
	\label{definition:linear_subgaussian}
	We call a distribution $p,\E[p]=0$ to be $\rho^2$-subgaussianwhen there exists $\rho > 0$ such that the random vector $\bar{\bx} \sim \bar p$ is $\rho^2$-subgaussian. $\bar p$ is the whitening of $p$ such that $\bar\bx\sim \bar p$ is equivalent to $\bx=\Sigma^{1/2}\bar\bx \sim p $, where $\Sigma=\E_p[\bx\bx^\top]$.  
	\footnote{A random vector $\bx$ is called $\rho^2$-subgaussian if for any fixed unit vector $\bv$ of the same dimension, the random variable $\bv^\top \bx$ is $\rho^2$-subgaussian, i.e., $\E[e^{s\cdot\bv^\top(\bx-\E[\bx])}]\le e^{s^2\rho^2/2}$ ($\forall s\in\R$).}
\end{definition}
Note that $\rho$ is defined on the whitening of the data. It doesn't scale with $\|\Sigma\|_{op}$ and should be viewed as universal constant. 

\begin{theorem}
	\label{thm:random_design_target}
	Fix a failure probability $\delta\in(0, 1)$.
	Suppose target distribution $p_T$ is $\rho^2$-subgaussian, and the sample size in target domain satisfies $n_U \gg \rho^4(d+\log\frac{1}{\delta})$. Let $\hat \bbeta:\by_S\rightarrow \hat C\hat\Sigma_S^{-1}X_S^\top \by_S$ where $\hat C$ is defined from Eqn. \eqref{eqn:cov_shift_alg_finite_n_U}. 
	Then with probability at least $1-\delta$ over the unlabeled samples from target domain, and for each fixed $X_S$ from source domain, our learned estimator $\hat\bbeta(\by_S)$ satisfies:
	\begin{equation}
	L_{\cB}(\hat\bbeta)\leq (1+O(\sqrt{\frac{\rho^4(d+\log(1/\delta))}{n}})) R_L(\cB).
	\end{equation}	
	When $\Sigma_T$ commutes with $\hat \Sigma_S$ or is rank 1, we have:
	\begin{equation}
	L_{\cB}(\hat\bbeta)\leq (1.25+O(\sqrt{\frac{\rho^4(d+\log(1/\delta))}{n}})) R_N(\cB).
	\end{equation}		 
\end{theorem}
Similarly all other results in the paper could be extended to $\hat\bbeta\leftarrow \argmin \hat L_{\cB}(\cdot)$, the estimator obtained with finite target samples $X_U$.

\begin{remark}[Incorporating the randomness from source data] 
For linear estimators, it naturally considers $X_S$ as fixed and Theorem \ref{thm:1.25minimax_risk} is comparing our estimator with the optimal nonlinear estimator using the same data $X_S$ from the source domain. In Appendix \ref{appendix:random_design}, we compare our estimator with an even stronger linear estimator with infinite access to $p_S$ and show that our estimator is still within multiplicative factor of it. 
\end{remark}

\subsection{Utilize source and target labeled data jointly}
\label{sec:linear_source_and_target_data}
In some scenarios, we have moderate amount of labeled data from target domain as well. In such cases, it is important to utilize the source and target labeled data jointly. Let $\by_S = X_S \bbeta^* + \bz_S$, $\by_T = X_T \bbeta^* + \bz_T$. We consider $X_S,X_T$ as deterministic variables, $\hat \Sigma_S^{-1} X_S^\top \by_S/n_S \sim \cN(\bbeta^*, \frac{\sigma^2}{n_S}\hat \Sigma_S^{-1} ) $ and $\hat \Sigma_T^{-1} X_T^\top \by_T /n_T \sim \cN(\bbeta^*, \frac{\sigma^2}{n_T}\hat \Sigma_T^{-1} ). $ Therefore conditioned on the observations $\by_S,\by_T$, a sufficient statistic for $\bbeta^* $ is $\hat\bbeta_{\tSS} := (n_S\hat \Sigma_S + n_T\hat \Sigma_T)^{-1}( X_S^\top \by_S + X_T^\top \by_T ).  $

\begin{claim}
	\label{claim:sufficient_statistic}
	$\hat\bbeta_{\tSS}$ is an unbiased sufficient statistic of $\bbeta^*$ with samples $\by_S,\by_T$. $\hat \bbeta_{\tSS} \sim \cN(\bbeta^*, \sigma^2(n_S\hat \Sigma_S + n_T \hat \Sigma_T)^{-1} )$. 
\end{claim}	
	
\paragraph{Algorithm:} First consider the estimator $\hat\bbeta_{\tSS} = (n_S\hat \Sigma_S + n_T\hat \Sigma_T)^{-1}( X_S^\top \by_S + X_T^\top \by_T )$. Next find the best linear function of $\hat\bbeta_{\tSS}$:
\begin{align*}
\hat\bbeta_{\tMM}
= & \arg\min_{C,\tau} r^2 \tau + \sigma^2 \Trace((n_S\hat \Sigma_S + n_T \hat \Sigma_T)^{-1} C^\top \Sigma_T C),\\
 \text{ s.t. } & (C-I)^\top\Sigma_T(C-I) \preceq \tau. 
\end{align*}


\begin{proposition}
\label{claim:linear_in_SS} 
The minimax estimator $\hat \bbeta_{\tMM}$ is of the form $C \hat \bbeta_{\tSS}$ for some $C$. When choosing $C$ with our proposed algorithm and when $\hat\Sigma_S$ commutes with $\hat\Sigma_T$ and $\Sigma_T$, we achieve the minimax risk $R_L(\cB) \leq 1.25 R_N(\cB)$. 
\end{proposition}

\section{Covariate shift with approximation error}
\label{sec:nonlinear_model}
Now we consider observations coming from nonlinear models: $\by_S=f^*(X_S)+\bz$. Let $\bbeta_S^* = \argmin_{\bbeta}\E_{\bx\sim p_S,z\sim \cN(0,\sigma^2)}[(f^*(\bx)+z-\bbeta^\top \bx)^2]$, and similarly for $\bbeta_T^*$. Notice now even with $f^*$ unchanged across domains, the input distribution affects the best linear model. Approximation error on source domain is $a_S(\bx) := f^*(\bx) - \bx^\top \bbeta_S^*$ and vice versa for $a_T$. 

Define the reweighting vector  $\bw\in \R^n$ as $w_i=p_T(\bx_i)/p_S(\bx_i)$. We form an unbiased estimator via 
\begin{align*}
\hat \bbeta_{\tLS} = & \argmin_\bbeta \{\sum_i  \frac{p_T(\bx_i)}{p_S(\bx_i)}( \bbeta^\top \bx_i - y_i)^2\}\\
= & (X_S^\top \diag(\bw) X_S)^{-1} (X_S^\top \diag(\bw) \by_S). 
\end{align*}

\begin{claim}
	\label{claim:MVUE}
	$\hat \bbeta_{\tLS}$ is asymptotically unbiased and normally distributed with covariance matrix $M:=\Sigma_T^{-1} \E_{\bx\sim p_T} [\frac{p_T(\bx)}{p_S(\bx)}  (a_T(\bx)^2+\sigma^2)\bx\bx^\top]\Sigma_T^{-1}$:
	\begin{align*}
	& \sqrt{n_S}(\hat \bbeta_{\tLS} - \bbeta_T^*) \overset{d}{\rightarrow} \cN(0, M ).	
	\end{align*}
\end{claim}
Note that large importance weights greatly inflates the variance of the estimator, especially when $p_T/p_S$ blows up somewhere. Therefore here we design the an algorithm to cope with the inflated variance. Again we want to minimize the worst case risk:
\begin{align*}
&  \min_{\hat \bbeta=C\hat \bbeta_{\tLS}} \max_{\bbeta_T^*\in\cB} \E \|\Sigma^{1/2}_T(\hat \bbeta - \bbeta_T^*)\|^2 \\
\overset{d}{\rightarrow} & \min_{C} \max_{\|\bbeta_T^*\|\leq r} \left\{\|\Sigma_T^{1/2}(C-I)\bbeta_T^* \|^2_2 + \frac{1}{n_S} \Trace(C M C^\top\Sigma_T) \right\}\\
= & \min_{C} \left\{  \|(C-I)^\top\Sigma_T (C-I)\|_2 r^2   +  \frac{1}{n_S} \Trace(C M C^\top \Sigma_T) \right\}
\end{align*} 
With $\hat\beta_{\tLS}$ computed beforehand, one could first estimate $M$ by let $\hat M:=\frac{1}{n_S} \sum_{i} \Sigma_T^{-1} \frac{p_T^2(\bx)}{p_S^2(\bx)} (y_i- \bx_i^\top \hat \bbeta_{\tLS} )^2\bx_i\bx_i^\top \Sigma_T^{-1}$. 
Therefore our estimator is $\hat \bbeta_{\tMM}\leftarrow \hat{C}\hat\bbeta_{\tLS}$, where $\hat C$ finds
\begin{align}
\label{eqn:best_linear_estimator_with_approx_error}
\hat C &\leftarrow  \arg\min_{\tau, C} \left\{ r^2 \tau + \frac{1}{n_S} \Trace( C \hat M C^\top \Sigma_T ) \right\} \\
\notag 
& \text{ s.t. }  (C-I)^\top \Sigma_T (C-I) \preceq \tau I.
\end{align}
\begin{claim}
	\label{claim:nonlinear_form}
	Let $\cB=\{\bbeta| \|\bbeta\|\leq r \}$, and $f^*\in \cF$ is some compact symmetric function class: $f\in \cF\Leftrightarrow -f\in \cF$. Then linear minimax estimator is of the form $C\hat \bbeta_{\tLS}$ for some $C$. When $\hat C$ solves Eqn.  \eqref{eqn:best_linear_estimator_with_approx_error}, $L_{\cB}(\hat \bbeta_{\tMM})$ asymptotically matches $R_L(\cB)$, the linear minimax risk.  
\end{claim}
By reducing from  $\by_S$ to $\hat\bbeta_{\tLS}$ we eliminate $n-d$ dimensions, and this claim says that $X_S^\top\by_S$ is sufficient to predict $\bbeta^*_T$. We note that $f^*$ is more general than a linear function and therefore the lower bound could only be larger than $R_N(\cB)$ defined in the previous section.

\subsection{Estimating \texorpdfstring{$p_T(\bx)/p_S(\bx)$}{TEXT}  }
\label{sec:estimate_density_ratio}
Even though estimating $p_T(\bx)/p_S(\bx)$ might be sample inefficient, it only involves unlabeled data and therefore instance weighting related algorithms still attract prior studies as demonstrated in the related work section. Practical ways to estimate the density ratio involve respectively estimating $p_T$ and $p_S$ \citep{lin2002support,zadrozny2004learning}, kernel mean matching (KMM) \cite{huang2006correcting}), and some common divergence minimization between weighted source distribution and target distribution \citep{sugiyama2008direct,sugiyama2012density,uehara2016generative,menon2016linking,kanamori2011f}.
 We propose another simple algorithm that is very convenient to use. 

We conduct regression on the data samples $(\bx,y)\sim q(\bx,y)$ where $q_Y(y)$ is Bernouli($\frac12$)\footnote{The scalar $1/2$ should be adjusted based on the number of unlabeled samples from source and target domain.} and $q_{X|Y}(\bx|y=1)= p_T$, $q_{X|Y}(\bx|y=0)= p_S$. Empirically, we will concatenate $X_S$ and $X_U$ to form input data and stack $\textbf{0}\in \R^{n_S}$ and $\textbf{1}\in \R^{n_U}$ as the target vector $\by$. 

\begin{proposition}
	\label{prop:estimate_density_ratio} 
The optimal function that solves $\alpha\leftarrow \argmin_f \E_{\bx,y\sim q} (f(\bx)-y)^2$	satisfies:
$\alpha(\bx) = \frac{p_T(\bx)}{p_S(\bx)+p_T(\bx)}$. 
\end{proposition}
Therefore with proper transformation\footnote{Apply $f(x) \to\frac{1}{1/f(x) -1}$}  on $\alpha$ one could get the importance weights. In practice, one might be flexible on choosing the function class $\cF$ for estimating $\alpha$ and sample complexity will be bounded by some standard measure of $\cF$'s complexity, e.g., Rademacher or Gaussian complexity \citep{bartlett2002rademacher}. Unlike KMM, this parametrized estimation applies to unseen data $\bx$ which makes cross-validation possible.  

\section{Near minimax estimator with model shift} 
\label{sec:model_shift} 
The general setting of transfer learning in linear regression involves both model shift and covariate shift. Namely, the generative model of the labels might be different:
$\by_S= X_S\bbeta_S^* + \bz_S$, and $\by_T = X_T \bbeta_T^* + \bz_T$. Denote by $\bdelta := \bbeta_S^*- \bbeta_T^*$ as the model shift. We are interested in the minimax linear estimator when $\|\bdelta\|\leq \gamma$ and $\|\bbeta_T^*\|\leq r$. Thus our problem becomes to find minimax estimator for $\bbeta^*_T\in \cB = \{\bbeta|\|\bbeta\|\leq r \} $ from $\by_S,\by_T$.

\paragraph{Algorithm: } First consider a sufficient statistic $(\bar \bbeta_{S}, \bar \bbeta_{T})$ for $(\bbeta_T^*,\bdelta)$. Here $\bar \bbeta_{S} = \hat \Sigma_S^{-1} X_S^\top \by_S/n_S\sim \cN(\bbeta_T^*+\bdelta, \frac{\sigma^2}{n_S}\hat \Sigma_S^{-1})$, and $ \bar \bbeta_T = \hat \Sigma_T^{-1} X_T^\top \by_T/n_T \sim  \cN(\bbeta_T^*, \frac{\sigma^2}{n_T}\hat \Sigma_T^{-1})$. Then consider the best linear estimator on top of it: $\hat \bbeta = A_1 \bar \bbeta_S + A_2 \bar \bbeta_T$. Write $\Delta = \{\bdelta|\|\bdelta\|\leq \gamma \}$ and  $L_{\cB,\Delta}(\hat \bbeta):= \max_{\bbeta_T^*\in\cB,\bdelta\in \Delta } \|\Sigma^{1/2}_T(\hat \bbeta - \bbeta_T^*)\|^2$. 

\begin{align}
\notag 
& R_L(\cB,\Delta) :=  \min_{\hat \bbeta=A_1\bar \bbeta_S+A_2\bar \bbeta_T} L_{\cB,\Delta}(\hat \bbeta) \\
\notag 
\leq & \min_{A_1,A_2} \max_{\|\bbeta_T^*\|\leq r, \|\bdelta\|\leq \gamma} \left\{2\|\Sigma_T^{1/2}((A_1+A_2-I)\bbeta_T^*\|^2 \right.\\
\label{eqn:am-gm_inequality}
& \left. + 2\|\Sigma_T^{1/2}A_1\bdelta \|^2  + \frac{\sigma^2}{n_S}\Trace(A_1 \hat \Sigma_S^{-1} A_1^\top) \right. \\
& \left. + \frac{\sigma^2}{n_T}\Trace(A_2 \hat \Sigma_T^{-1} A_2^\top) \right\} \tag{AM-GM} \\
\notag 
= & \min_{A_1,A_2} \left\{2\|\Sigma_T^{1/2}((A_1+A_2-I)\|_2^2 r^2 + 2\|\Sigma_T^{1/2}A_1\|^2_2\gamma^2 \right.\\
\notag & \left. + \frac{\sigma^2}{n_S}\Trace(A_1 \hat \Sigma_S^{-1} A_1^\top) + \frac{\sigma^2}{n_T}\Trace(A_2 \hat \Sigma_T^{-1} A_2^\top) \right. \\
&\left. =: r_{\cB,\Delta}(A_1,A_2) \right\}. 
\end{align} 
Therefore we optimize over this upper bound and reformulate the problem as a convex program:
\begin{align}
\notag 
(\hat A_1, \hat A_2) & \leftarrow \argmin_{A_1,A_2, a, b} \left\{ 2a r^2 + 2b\gamma^2 \right.\\
\notag 
& \left.  + \frac{\sigma^2}{n_S}\Trace(A_1 \hat \Sigma_S^{-1} A_1^\top) + \frac{\sigma^2}{n_T}\Trace(A_2 \hat \Sigma_T^{-1} A_2^\top) \right\}\\
\notag 
\text{s.t.}& (A_1+A_2-I)^\top \Sigma_T(A_1+A_2-I) \preceq a I,\\
\label{eqn:convex_program_changed_beta}
& A_1^\top \Sigma_T A_1\preceq bI.
\end{align}
Our estimator is given by:  $\hat \bbeta_{\tMM} = \hat A_1\bar \bbeta_S + \hat A_2 \bar \bbeta_T.$ Since $\hat \bbeta_{\tMM}$ is a relaxation of the linear minimax estimator, it is important to understand how well $\hat \bbeta_{\tMM}$ performs on the original objective:
\begin{claim}
	\label{claim:beta_change_relaxed_loss}
	$ R_L(\cB,\Delta)\leq L_{\cB,\Delta}(\hat \bbeta_{\tMM})\leq  2R_L(\cB,\Delta)$.
\end{claim}
Finally we show with the relaxation we still achieve a near-optimal estimator even among all nonlinear rules.
\begin{theorem}
	\label{thm:27minimax_model_shift}
When $\Sigma_T$ commutes with $\hat\Sigma_S$,	
it satisfies:
\begin{align*}
L_{\cB,\Delta}(\hat \bbeta_{\tMM}) := & \max_{\bbeta_T^*\in \cB,\bdelta \in \Delta} \|\Sigma_T^{1/2}(\hat \bbeta_{\tMM} - \bbeta^*_T)\|^2\\
 \leq & 27 R_N(\cB,\Delta).
\end{align*}
Here $R_N(\cB,\Delta):= \min_{\hat \bbeta(\by_S,\by_T)}\max_{\bbeta_T^*\in \cB,\bdelta \in \Delta} \|\Sigma_T^{1/2}(\hat \bbeta - \bbeta^*_T)\|$ is the minimax risk.	
\end{theorem}
We defer the complete proof to the appendix. The main proof technique is to decompose the problem to 2-$d$ sub-problems with closed-form solutions and are solvable with Le Cam's two point lemma. We include the proof sketch here: 

\begin{proof}[Proof sketch of Theorem \ref{thm:27minimax_model_shift}]
	For the ease of understanding, we provide a simple proof sketch when $\Sigma_S=\Sigma_T$ are diagonal.  We first define the hardest hyperrectangular subproblem. Let $\cB(\btau)=\{\bb:|\beta_i| \leq \tau_i\}$ be a subset of $\cB$ and similarly for $\Delta(\bzeta)$. We show that $R_L(\cB,\Delta)=\max_{\btau\in \cB,\bzeta\in\Delta} R_L(\cB(\btau),\Delta(\bzeta))$, and clearly $ R_N(\cB,\Delta)\geq \max_{\btau\in \cB,\bzeta\in\Delta} R_N(\cB(\btau),\Delta(\bzeta)) $. Meanwhile we show when the sets are  hyperrectangles the minimax (linear) risk could be decomposed to 2-d problems: $R_L(\cB(\btau),\Delta(\bzeta))=\sum_i R_L(\tau_i,\zeta_i)$. Each $R_L(\tau_i,\zeta_i)$ is the linear minimax risk to estimate $\beta_i$ from $x\sim \cN(\beta_i+\delta_i,1)$ and $y\sim \cN(\beta_i,1)$ where $ |\beta_i|\leq \tau_i$ and $|\delta_i|\leq \zeta_i$. This 2-d problem for linear risk has a closed form solution, and the minimax risk can be lower bounded using Le Cam's two point lemma. We show $R_L(\tau_i,\zeta_i)\leq 13.5 R_N(\tau_i,\zeta_i)$ and therefore: 
	\begin{align*}
	\frac{1}{2}L_{\cB,\Delta}(\hat\bbeta_{\tMM})& \overset{\text{Claim }\ref{claim:beta_change_relaxed_loss}}{\leq}  R_L(\cB,\Delta)\\
	\overset{\text{Lemma } \ref{lemma:relate_to_worst_hyperrectangle}}{=}&\max_{\btau\in \cB,\bzeta\in\Delta} R_L(\cB(\btau),\Delta(\bzeta))\\
	\overset{\text{Prop }\ref{proposition:decompose}.a}{=} &\max_{\btau\in \cB,\bzeta\in\Delta} \sum_i R_L(\tau_i,\zeta_i) \\
	\overset{\text{Lemma }\ref{lemma:1d_bound_with_beta_change}}{\leq}& \max_{\btau\in \cB,\bzeta\in\Delta} 13.5 \sum_i R_N(\tau_i,\zeta_i) \\
	\overset{\text{Prop }\ref{proposition:decompose}.b}{=} & 13.5 \max_{\btau\in \cB,\bzeta\in\Delta} R_N(\cB(\btau),\Delta(\bzeta))\\
	\leq & 13.5 R_N(\cB,\Delta). 
	\end{align*} 
\end{proof}



\begin{figure*}[htbp]
	\begin{tabular}{ccc}
\hspace{-0.4cm}
		\includegraphics[width=0.37\linewidth]{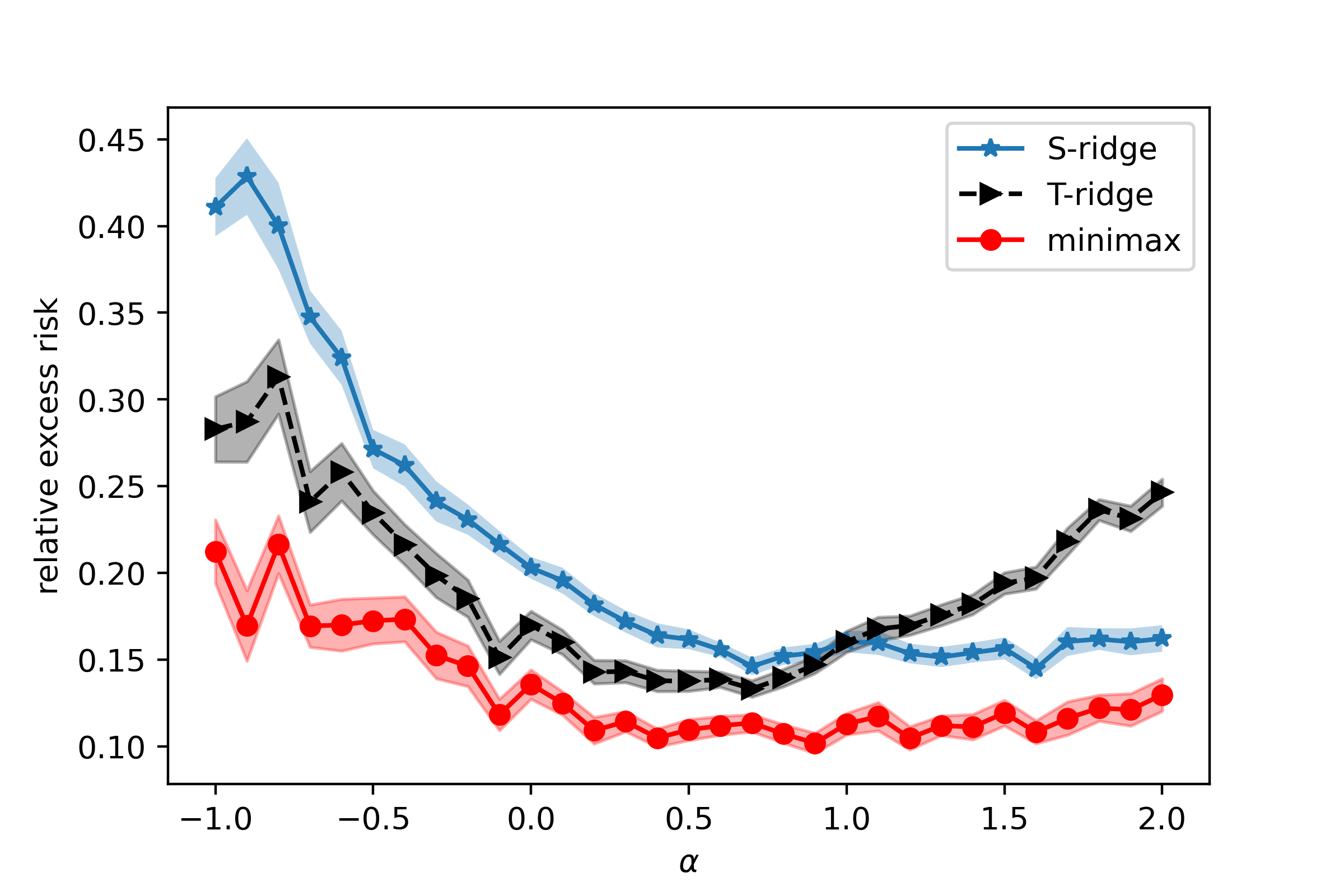} & \hspace{-1.1cm}
		\includegraphics[width=0.37\linewidth]{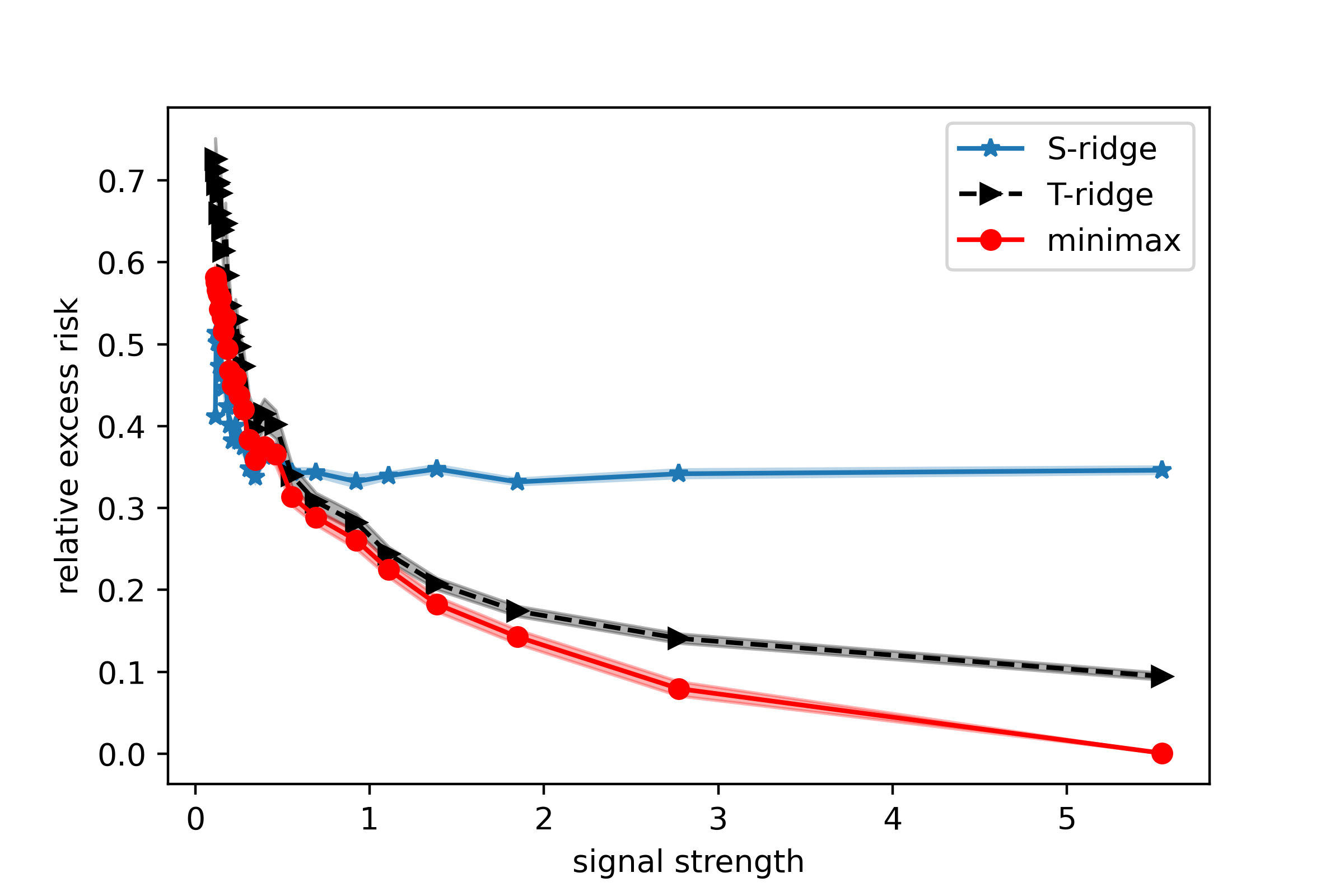} & \hspace{-1.1cm}
		\includegraphics[width=0.33\linewidth]{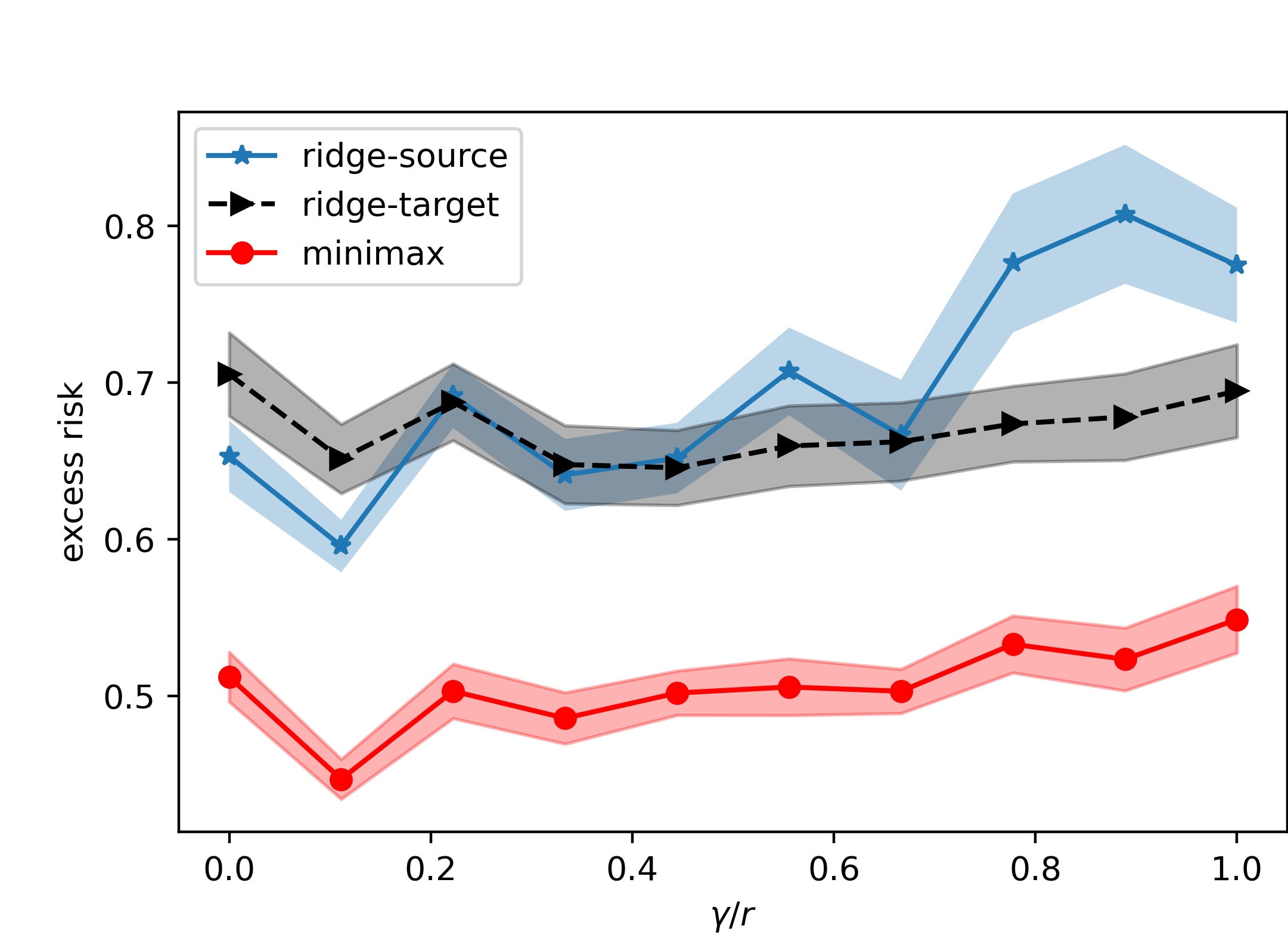} \hspace{-1cm}
		\\
		(a)	covariate eigen-spectrum & (b) signal strength & (c) model shift
	\end{tabular}
	\caption{\emph{Performance comparisons.} (a): The x-axis $\alpha$ defines the spread of eigen-spectrum of $\Sigma_S$:  $s_i \propto 1/i^{\alpha}, t_i\propto 1/i$. (b) x-axis is the normalized value of signal strength: $\|\Sigma_T\bbeta^*\|/r$. 
(c) X-axis is the model shift measured by $\gamma/r$. Performance with standard error bar is from 40 runs.}
	\label{figure:alpha_beta_shift} 
\end{figure*}
\section{Experiments} 

Our estimators are provably near optimal for the worst case $\bbeta^*$. However, it remains unknown whether on average they outperform other baselines. With synthetic data we explore the performances with random $\bbeta^*$. We are also interested to investigate the conditions when we win more. 

\textbf{Setup. }
We set $n_S=2000, d=50, \sigma=1, r=\sqrt{d}$. For each setting, we sample $\bbeta^*_T$ from standard normal distribution and rescale it to be norm $r$. We estimate $\Sigma_T$ by $n_U=2000$ unlabeled samples. We compare our estimator with ridge regression (S-ridge) and a variant of ridge regression transformed to target domain (T-ridge): $\hat\bbeta^\lambda_{\tRR,T}=\argmin \frac{1}{n}\|\Sigma_T^{1/2}(\bbeta- \hat\Sigma_S^{-1}X_S^\top\by_S)\|^2+\lambda \|\bbeta\|^2 = (\Sigma_T+\lambda I)^{-1}\Sigma_T \hat\bbeta_{\tSS}$. 

\textbf{Covariate shift. }In order to understand the effect of covariate shift on our algorithm, we consider three types of settings, each with a unique varying factor that influences the performance: 1) covariate eigenvalue shift with shared eigenspace; 2) covariate eigenspace shift with fixed eigenvalues\footnote{We leave this result in appendix since performance appears invariant to this factor.}; 3) signal strength change. We also have an additional $200$ labeled data from target domain as validation set only for hyper-parameter tuning. 

\textbf{Model shift. }Next we consider the problem with model shift. We sample a random $\bdelta$ with norm $\gamma$ varying from $0$ to $r=\sqrt{d}$ and observe data generated by $\by_S=X_S(\bbeta_T^*+\bdelta)+\bz_S\in \R^{2000},\bz_S\sim \cN(0,I)$ and $\by_T=X_T\bbeta_T^*+\bz_T\in \R^{500},\bz_T\sim \cN(0,I)$. We compare our estimator with two baselines: "ridge-source" denotes ridge regression using only source data, and "ridge-target" is from ridge regression with target data.

Figure \ref{figure:alpha_beta_shift} demonstrates the better performance of our estimator in all circumstances. From (a) we see 
that with more discrepancy between $\Sigma_S$ and $\Sigma_T$, our estimator tends to perform better. (b) shows our estimator is better when the signal is relatively stronger. From (c) we can see that with the increasing model shift measured by $\gamma/r$, S-ridge becomes worse and is outperformed by T-ridge that remains unchanged. Our estimator becomes slightly worse as well due to the less utility from source data, but remains the best among others. 
When $\gamma/r\approx 0.2$, our method has the most improvement in percentage compared to the best result among ridge-source and ridge-target.


\subsection{Experiments with approximation error}
\begin{minipage}{0.48\textwidth} 
	\includegraphics[width=\columnwidth]{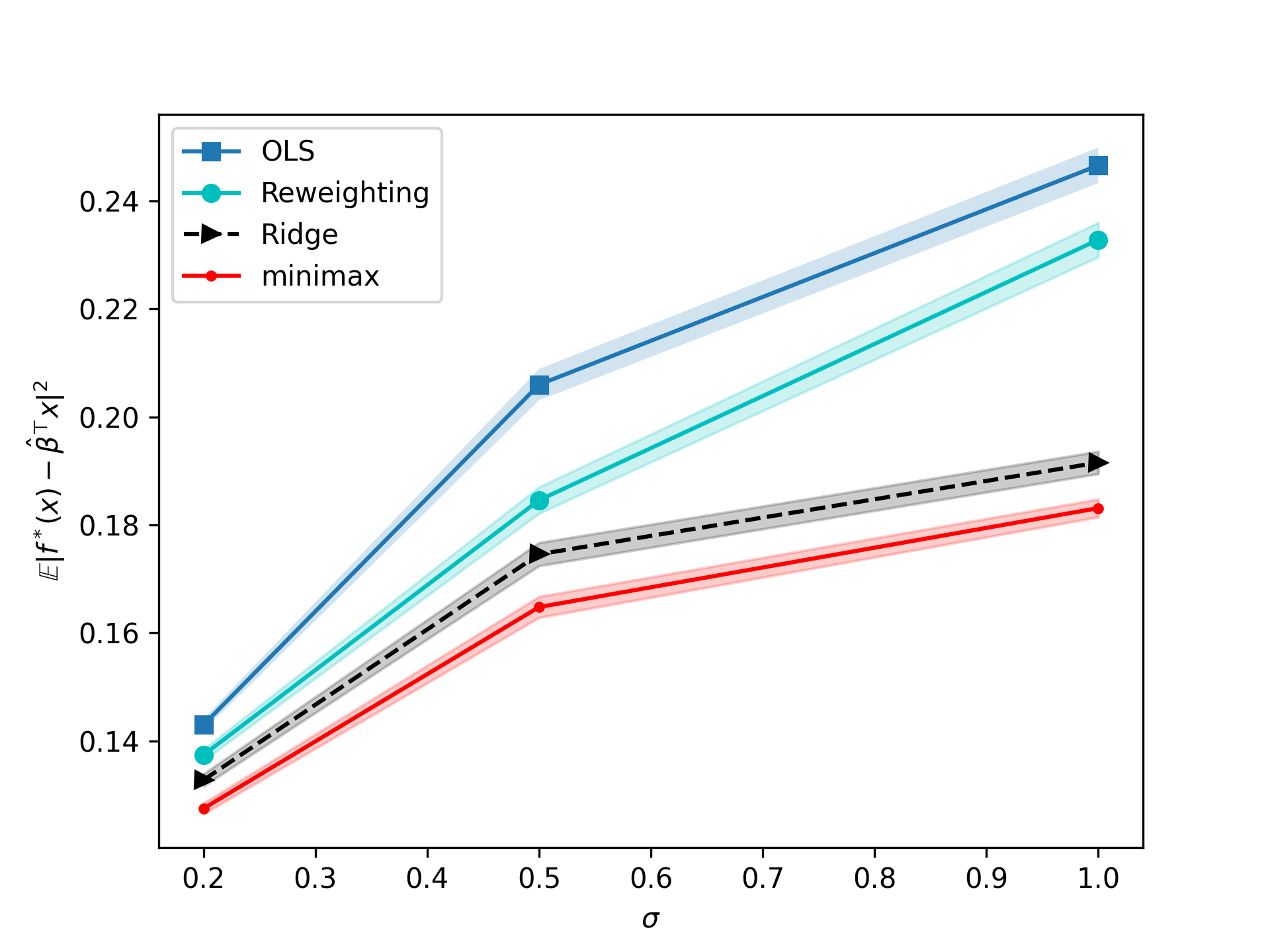}
	\captionof{figure}{The x-axis is noise level $\sigma$ and y-axis is the excess risk (with approximation error). Performance with standard error bar is from 40 runs.  }
	\label{fig:nonlinear}
\end{minipage}\hspace{0.03\textwidth}
\begin{minipage}{0.48\textwidth}
	\includegraphics[width=\columnwidth]{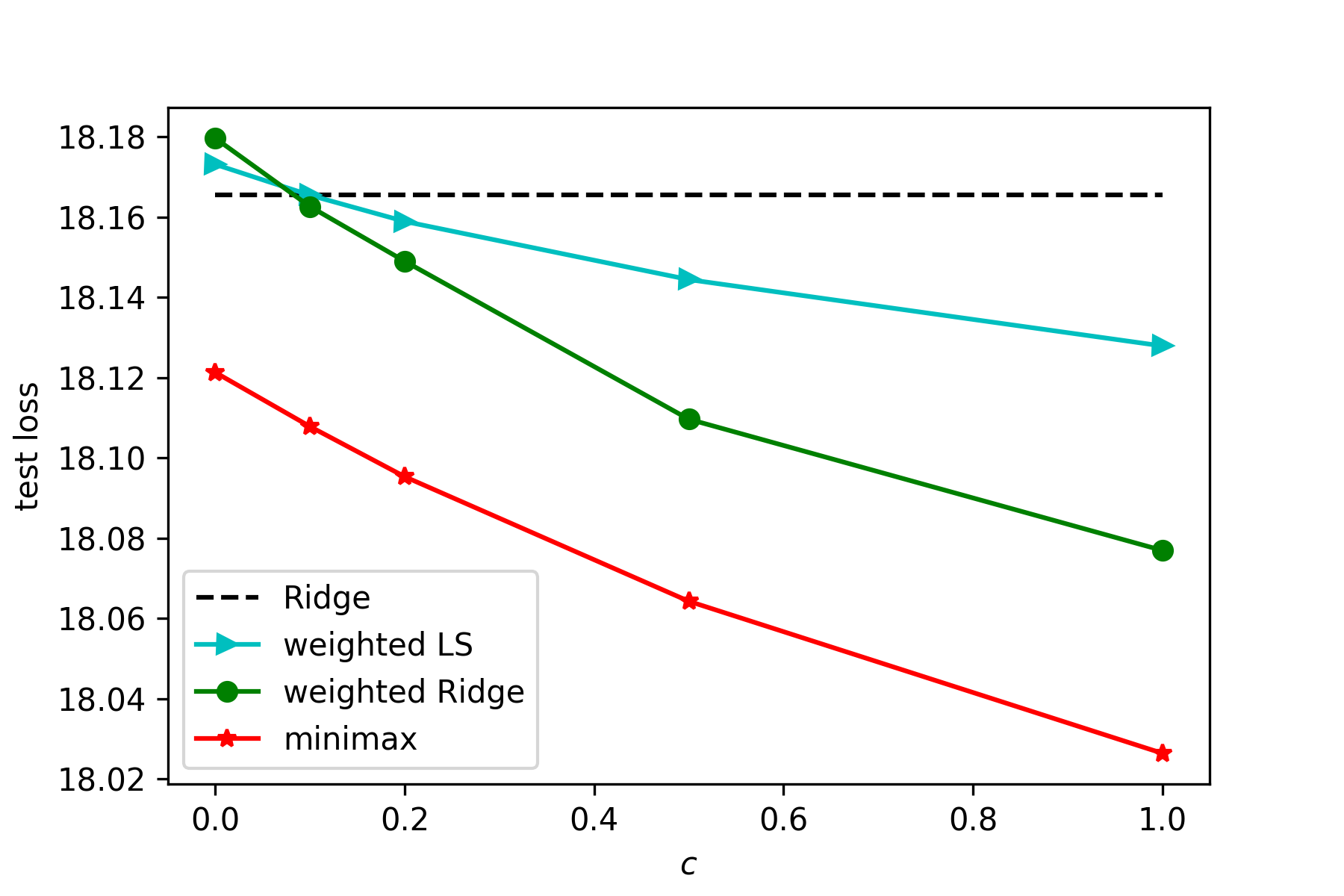}
	\captionof{figure}{Comparisons on Yearbook Dataset \cite{ginosar2015century}. The $x$-axis is the reweighting strength. }
	\label{fig:yearbook}
\end{minipage}

Finally, we conduct empirical studies with nonlinear models. 
We maintain the same setting as before. We also generate a small validation dataset from target domain: $X_{\text{CV}}\in \R^{500\times 50}, $ sampled from $\cN(0,\Sigma_T)$, $\by_{\text{CV}}=f^*(X_{\text{CV}})+\bz_{\text{CV}}$, with $\bz_{\text{CV}}\sim \cN(0,\sigma^2 I)$. We choose $\lambda_i(\Sigma_S)\propto i, \lambda_i(\Sigma_T)\propto 1/i$, and the eigenspace for both $\Sigma_S$ and $\Sigma_T$ are random orthonormal matrices. ($\|\Sigma_S\|_F^2=\|\Sigma_T\|_F^2=d.$) The ground truth model is a one-hidden-layer ReLU network: $f^*(\bx)=1/d\ba^\top (W\bx)_{+}$, where $W$ and $\ba$ are randomly generated from standard Gaussian distribution. We observe noisy labels: $\by_S=f^*(\bx)+\bz$, where $z_i\sim \cN(0,\sigma^2)$.

\paragraph{Estimating weights $p_T(\bx)/p_S(\bx)$. }
Since the generated data samples are Gaussian, the absolute weights for $p_T(\bx)/p_S(\bx) = \sqrt{\frac{|\Sigma_S|}{|\Sigma_T|}}\exp(\frac{1}{2}\bx^\top(\Sigma_S^{-1}-\Sigma_T^{-1})\bx ) $. However, this absolute value scales exponentially with the  norm of $\bx$ and can amplify the variance. Meanwhile, when one multiplies both $X_S,\by_S$ by 10, the ground truth $\beta$ doesn't change but the absolute value for $p_T(\bx)/p_S(\bx)$ will change drastically. This discrepancy highlights the importance of relative magnitudes (among samples) instead of the absolute value, as noted by \citet{kanamori2009least}. 

To obtain a relative score, we first estimate the absolute density ratio $\alpha(\bx)\approx p_T(\bx)/(p_S(\bx)+p_T(\bx))$ following our algorithm in Section \ref{sec:estimate_density_ratio} with linear regression. 
We then uniformly assign the weight $w_i$ for each sample by 10 discrete values $1,2,3\cdots 10$ based on the absolute value of $\alpha(\bx)$ and then rescale the reweighting vector properly. We use the conventional way to adjust the reweighting strength by using $w_i^{c},c\in [0,1]$ and choose $c$ by cross validation.  

We implement our method (Eqn. \ref{eqn:best_linear_estimator_with_approx_error}) using the estimated weights as above, and plot the excess risk comparisons in Figure \ref{fig:nonlinear}. The baselines we choose are ordinary least square ("OLS" in Figure \eqref{fig:nonlinear}), ridge regression (Legend is "Ridge") and weighted least square \cite{kanamori2009least} (Legend is "Reweighting"; $\hat\bbeta_{\text{LS}}$ in our main text). For ridge regression, reweighting and our methods, we tune hyperparameters through cross-validation. All results are presented from 40 runs where the randomness comes from $f^*$ and the eigenspaces of $\Sigma_S,\Sigma_T$. From Figure \ref{fig:nonlinear} we could see that reweighting algorithm improves over ordinary least square but is outperformed by ridge regression due to large variance. Our algorithm achieves the best performance among others by appropriately reweighting then reducing the variance.

\paragraph{Experiments on Berkeley Yearbook Dataset}
To verify the performance of our algorithm on real-world data, we conduct an experiment on the Berkeley Yearbook dataset \cite{ginosar2015century}. We randomly split the data to form source and target tasks, where the source has 63.2\% male photos and 43.4\% male images for the target task. Input $X$ is gray-scale portraits, and $Y$ is the year the photo is taken (ranging from $1905$ to $2013$). We implement our algorithms together with the baselines and estimate the density ratio from the data. 
We demonstrate the performance improvement in Figure \ref{fig:yearbook}.  The $x$-axis is the scalar that adjusts reweighting strength $c$ defined in the previous paragraph.

\section{Conclusion}
We  study in depth the minimax linear estimator for linear regression under various distribution shift settings. We investigated the optimal linear estimators with covariate shift for linear models in unsupervised and supervised domain adaptation settings, with no or scarce labeled data from the target distribution. For nonlinear models with approximation error, we also introduce the minimax linear estimator together with an easy-to-use density ratio estimation method. We further explore some moderate model shift in the linear setting. Our estimators achieve near-optimal worst-case excess risk measured on the target domain and, in some circumstances, are within constant of the minimax risk among all nonlinear rules. The significant improvement of our estimators over ridge regression is demonstrated by a theoretical separation result and by empirical validations even for average case with random parameters. 

In future work, we will extend our algorithm to classification problems under distribution shift and apply the algorithms to  fine-tuning the last-layer of a deep network.

\section*{Acknowledgements}

QL was supported by NSF \#2030859 and the Computing Research Association for the CIFellows Project.
WH was supported by NSF, ONR, Simons Foundation, Schmidt Foundation, Amazon Research, DARPA and SRC.
JDL was supported by ARO under MURI Award W911NF-11-1-0303, the Sloan Research Fellowship, and NSF CCF 2002272.

\bibliographystyle{plainnat}

\bibliography{iclr2021_conference,references2}

\appendix

\section{Omitted proof for minimax estimator with covariate shift}

\subsection{Pinsker's Theorem and covariate shift with linear model} 

\begin{theorem}[Pinsker's Theorem]
	\label{thm:pinsker}
	Suppose the obervations follow sequence model $ y_i = \theta_i^* + \epsilon_i z_i, \epsilon_i >0,i\in [d]$, and $\Theta$ is an ellipsoid in $\R^d$:
	$\Theta = \Theta(a, C) = \{ \theta: \sum_i a_i^2 \theta_i^2 \leq C^2  \}$. Then the minimax linear risk 
	\begin{align*}
	R_L(\Theta)  := & \min_{\hat\btheta \text{ linear}}\max_{\btheta^*\in \Theta} \E\|\hat\btheta(\by) - \btheta^* \|^2 \\
	=& \sum_i \epsilon_i^2 (1-a_i/\mu)_+,
	\end{align*} 
	where $\mu= \mu(C)$ is determined by 
	\begin{equation*}
	\sum_{i=1}^d \epsilon_i^2 a_i(\mu-a_i)_+ = C^2.
	\end{equation*}
	The linear minimax estimator is given by
	\begin{equation}
	\hat \theta^*_i(y) = c_i^* y_i = (1-a_i/\mu)_+ y_i,
	\end{equation}
	and is Bayes for a Gaussian prior $\pi_C$ having independent components $\theta_i \sim \cN(0,\tau_i^2)$ with $\tau_i^* = \epsilon_i^2 (\mu/a_i - 1)_+$. 
\end{theorem} 

Our theorem \ref{thm:linear_minimax_estimator} is to connect our parameter $\bbeta^*$ to the $\btheta^*$ in pinsker's theorem. First we show that reformulating the problem from a linear map of $n$ dimensional observations $\by_S$ to a linear map on the $d$-dimensional statistic $\hat\bbeta_{SS}$ is sufficient, i.e., Claim \ref{claim:A=A_1X^T}: 

\begin{proof}[Proof of Claim \ref{claim:A=A_1X^T}]
	This is to show that if $\hat\bbeta(\by_S):=A\by_S$ is a minimax linear estimator, each row vector of $A\in \R^{d\times n}$ is in the column span of $X_S$. Write $A= A_1 X_S^\top + A_2 W^\top$ where $W\in \R^{n\times (n-d)}$, columns of which forms the orthonormal complement for the column space of $X_S$. Equivalently we want to show $A_2=0$. We have
	\begin{align*}
	R_L(\cB) \equiv & \min_{\hat\bbeta=A\by} \max_{\bbeta^*\in \cB} \E\|\Sigma_{T}^{1/2}(\hat \bbeta -\bbeta^*)\|^2\\
	= & \min_{A_1, A_2} \max_{\bbeta^*\in \cB} \E\|\Sigma_{T}^{1/2}((A_1X_S^\top+A_2W^\top)\by_S -\bbeta^*)\|^2\\
	= & \min_{A_1, A_2} \max_{\bbeta^*\in \cB} \E\|\Sigma_{T}^{1/2}(A_1 X_S^\top (X_S \bbeta^* + \bz) + A_2 W^\top \bz -\bbeta^*)\|^2 \tag{Since $W^\top X_S=0$} \\
	= & \min_{A_1, A_2} \max_{\bbeta^*\in \cB}\left\{ \|\Sigma_{T}^{1/2}(A_1 X_S^\top X_S-I) \bbeta^*\|^2 + \E \|\Sigma_T^{1/2}A_1X_S^\top\bz\|^2\right. \\
	& \left. + \E \|\Sigma_T^{1/2}A_2 W^\top\bz\|^2  + \E\left\langle \Sigma_T^{1/2}A_1X_S^\top\bz, \Sigma_T^{1/2}A_2 W^\top\bz \right\rangle
	\right\} 
	\tag{Other cross terms vanish since $\E[\bz]=\bm0$} \\
	=& \min_{A_1, A_2} \max_{\bbeta^*\in \cB}\left\{ \|\Sigma_{T}^{1/2}(A_1 X_S^\top X_S-I) \bbeta^*\|^2 + \E \|\Sigma_T^{1/2}A_1X_S^\top\bz\|^2 + \E \|\Sigma_T^{1/2}A_2 W^\top\bz\|^2  ,
	\right\} 
	\end{align*} 
	where the last equation is because
	\begin{align*}
	&\E\left\langle \Sigma_T^{1/2}A_1X_S^\top\bz, \Sigma_T^{1/2}A_2 W^\top\bz \right\rangle
	= \E\left[ \Trace\left[\Sigma_T^{1/2}A_1X_S^\top\bz \bz WA_2^\top \Sigma_{T}\right]  \right] \\
	=\,& \Trace\left[\Sigma_T^{1/2}A_1X_S^\top \E[\bz \bz^\top] WA_2^\top \Sigma_{T}\right]
	= \sigma^2 \Trace\left[\Sigma_T^{1/2}A_1X_S^\top  WA_2^\top \Sigma_{T}\right]
	= 0.
	\end{align*}
	
	Clearly, at min-max point, without loss of generality we can take $A_2=0$. 
\end{proof}

Formally the proof for Theorem \ref{thm:linear_minimax_estimator} is presented here: 
\begin{proof}[Proof of Theorem \ref{thm:linear_minimax_estimator}]
To use Pinsker's theorem to prove Theorem \ref{thm:linear_minimax_estimator}, we simply need to transform the problem match its setting. 
Let $ \by_T = \Sigma_{T}^{1/2}\hat\Sigma_S^{-1}X_S^\top\by_S/n_S = \btheta_T^* + \bz_T$, where $\btheta_T^* = U^\top \Sigma_T^{1/2} \bbeta^*$ and  $\bz_T \sim \cN(0, \sigma^2\diag([t_i/s_i]_{i=1}^d)/n_S)$. The set for $\theta_T^*$ is $ \Theta=\{ \btheta | \|\Sigma_T^{-1/2}U \btheta\|\leq r\}$, i.e., $\Theta = \{\theta | \sum_i \btheta_i^2/t_i \leq r^2  \}$. 

Now with Pinsker's theorem, $\hat \btheta(\by_T)_i = (1-1/(\mu\sqrt{t_i}))_+ (y_T)_i  $ is the best linear estimator for $\btheta_T^*$, where $\mu=\mu(r)$ solves 
\begin{equation}
\frac{\sigma^2}{n_S}\sum_{i=1}^d \frac{\sqrt{t_i}}{s_i} (\mu-\frac{1}{\sqrt{t_i}})_+ = r^2.
\end{equation}
Connecting to the original problem, we get that the best estimator for $\Sigma_T^{1/2}\bbeta^* $ is $U(I-\frac{1}{\mu} \diag([1/\sqrt{t_i}]_{i=1}^d) ) \by_T = U(I-\frac{1}{\mu} \diag([1/\sqrt{t_i}]_{i=1}^d) ) U^\top \Sigma_T^{1/2}\Sigma_S^{-1} X_S^\top \by_S /n_S $. 

\end{proof}

\subsection{Omitted proof for noncommute second-moment matrices} 
\paragraph{Convex program. } Our estimator for $\bbeta^*$ can be achieved through convex programming:
\begin{proof}[Proof of Proposition \ref{prop:convex_program}]
	First note the objective function is quadratic in $C$ and linear in $\tau$, therefore we only need to prove the constraint $S=\{(C,\tau)|(C-I)^\top\Sigma_T (C-I) \preceq \tau I \}$ is a convex set. Notice for $(C_1,\tau_1), (C_2,\tau_2) \in S$, i.e., $(C_i-I)^\top \Sigma_T(C_i-I)\preceq \tau_i I, i\in \{1,2\} $. We simply need to prove for $C_\alpha:=\alpha C_1+(1-\alpha)C_2, \tau_\alpha:= \tau_1 \alpha + \tau_2 (1-\alpha)$, $(C_{\alpha}-I)^\top \Sigma_T (C_\alpha - I) \preceq \tau_\alpha I$ for any $\alpha \in [0,1]$.
	First, notice $(C_1-C_2)^\top \Sigma_T (C_1-C_2)\succeq 0$. Next,
	\begin{align*}
	& (C_{\alpha} -I)^\top \Sigma_T (C_{\alpha} -I) \\
	= & \alpha (C_1-I)^\top \Sigma_T (C_1-I) + (1-\alpha) (C_2-I)^\top \Sigma_T (C_2-I)\\
	& - \alpha(1-\alpha) (C_1-C_2)^\top \Sigma_T(C_1-C_2) \\
	\preceq &    \alpha (C_1-I)^\top \Sigma_T (C_1-I) + (1-\alpha) (C_2-I)^\top \Sigma_T (C_2-I) \\
	\preceq & \tau_\alpha I.
	\end{align*}	
\end{proof}

\paragraph{Benefit of our estimator.} Compared to ridge regression, our estimator could possibly achieve much better ($d^{-1/4}$) improvements:
\begin{proof}[Proof of Remark \ref{remark:order-wise_better}]
	We consider diagonal second-moment matrices $\hat \Sigma_S=\diag(\bs), \Sigma_T=\diag(\bt)$, $\sigma=1$. 
	First we calculate the expected risk obtained with ridge regression:  $\hat \bbeta_{\tRR}^\lambda = (X_S^\top X_S/n + \lambda I)^{-1} X_S^\top \by_S/n_S \sim \cN((\hat \Sigma_S+\lambda I)^{-1}\Sigma_S \bbeta^*, 1/n_S (\Sigma_S+\lambda I)^{-2}\Sigma_S )$.
	\begin{align*}
	L_{\cB}(\bbeta_{\tRR}^\lambda ) = &
	\max_{\bbeta^*\in \cB} \E_{\by_S}\|\Sigma_T^{1/2}( \hat \bbeta_{\tRR}^\lambda(\by_S) - \bbeta^*)\|^2\\ 
	= & \max_{\bbeta^*\in \cB} 
	\| \Sigma_T^{1/2} ((\hat\Sigma_S+\lambda I)^{-1} \hat \Sigma_{S}-I)\bbeta^* \|^2 + \Trace(\frac{1}{n_S}(\hat\Sigma_S+\lambda I)^{-2}\hat\Sigma_S\Sigma_T)\\
	= & \max_i r^2 \left(\frac{\sqrt{t_i}s_i }{s_i+\lambda} -\sqrt{t_i} \right)^2 + \sum_i \frac{1}{n_S} \frac{t_is_i}{(s_i+\lambda)^2 }.
	\end{align*}
	Compared to our risk:
	\begin{align*}
	R_L(\cB) = \sum_i \frac{1}{n_S} \frac{t_i}{s_i}(1-\frac{1}{\sqrt{t_i}\mu})_+, 
	\end{align*}
	where $\frac{1}{n}\sum_{i=1}^d \frac{\sqrt{t_i}}{s_i} (\mu-\frac{1}{\sqrt{t_i}})_+=r^2$. 
	Let $r^2=\frac{\sqrt{d}}{n_S}$,  $s_i=1,\forall i, t_i=1,\forall i\in [d_0], t_i = d^{-1/2}, d_0<i\leq d,$ where $d_0 = \frac{\sqrt{d}}{d^{1/4}-1}\approx d^{1/4}$. Then $\mu=1$, and $R_L(\cB)= \frac{d^{1/4}}{n}$.  
	In this case, 
	\begin{align*}
	& \min_{\lambda} \max_i r^2 \left(\frac{\sqrt{t_i}s_i }{s_i+\lambda} -\sqrt{t_i} \right)^2 + \sum_i \frac{1}{n_S} \frac{t_is_i}{(s_i+\lambda)^2 }\\
	= & \min_{\lambda} \max_i \frac{\sqrt{d}}{n} \left(\frac{\sqrt{t_i}}{1+\lambda} -\sqrt{t_i} \right)^2 + \sum_i \frac{1}{n_S} \frac{t_i}{(1+\lambda)^2 }
	\geq & \min_{\lambda} \frac{\sqrt{d}}{n} \frac{\lambda^2}{(1+\lambda)^2} + \frac{\sqrt{d}}{n}\frac{1}{(1+\lambda)^2}\\
	\geq & \frac{\sqrt{d}}{2n}.
	\end{align*}
	Therefore $ \min_{\lambda} L_{\cB}(\hat\bbeta_{\tRR}^{\lambda})\geq d^{1/4}R_L(\cB)/2$.
\end{proof}

\paragraph{Near minimax risk. } Even among all nonlinear estimators, our estimator is within 1.25 of the minimax risk:
\begin{proof}[Proof of Theorem \ref{thm:1.25minimax_risk}]
First we note that for both linear and nonlinear estimators, it is sufficient to use $\hat\bbeta_{\tSS}$ instead of the original observations $\by_S$. See Lemma \ref{lemma:sufficient_statistic_is_enough_for_best_estimator} and its corollary. Therefore it suffices to do the following reformulations of the problem. 	
	
When $\Sigma_S$ and $\Sigma_T$ commute, we formulate the problem as the following Gaussian sequence model. Recall $\hat\Sigma_S=U\diag(\bs)U^\top, \Sigma_T=U\diag(\bt)U^\top$. Let $\btheta^*= U^\top\Sigma_T^{1/2}\bbeta^*$, and $\by=U^\top\Sigma_T^{1/2}\hat\bbeta_{\tSS}\sim \cN(\btheta^*, \frac{\sigma^2}{n_S} \diag(\bt/\bs) )$. Our objective of minimizing $\|\Sigma_{T}^{1/2}(\hat\bbeta(\by_S)-\hat\bbeta^*) \|$ from linear estimator is equivalent to minimizing $\|U(\hat\btheta(\by)-\hat\btheta^*) \|= \|\hat\btheta(\by)-\hat\btheta^* \|$ from linear estimator. 

The set for the parameter that satisfies $\btheta^*=U^\top\Sigma_T^{1/2}\bbeta^*, \|\bbeta^*\|\leq r$ is equivalent to $ \|\Sigma_T^{-1/2}U\btheta^*\|\leq r \Leftrightarrow \|\theta_i^*/\sqrt{t_i}\|\leq r$ is an axis-aligned ellipsoid. 
Then we could directly derive our result from Corollary 4.26 from \cite{johnstone2011gaussian}. Note that this result is a special case of Theorem \ref{thm:27minimax_model_shift} and we have provided a detailed proof in Section \ref{section:model_shift_proof}. Therefore here we save further descriptions. 


For the case when $\Sigma_T=\ba\ba^\top$ is rank-1, the objective function becomes:
\begin{equation*}
R_L^*(\cB) = \min_{\bbeta^* \text{ linear}} \max_{\bbeta\in \cB} \E(\ba^\top(\hat\bbeta(\by_S)-\bbeta^* ))^2.
\end{equation*}

Then the result could be derived from Corollary 1 of
\cite{donoho1994statistical}, which reformulate the problem to the hardest one-dimensional problem which becomes tractable.

\end{proof}

In the proof above, we equate the best nonlinear estimator on $\by_S$ as the best nonlinear estimator on $\hat \bbeta_{\tSS}$. The reasoning is as follows:

\begin{lemma}[Sufficient statistic is enough to achieve a best estimator]
	\label{lemma:sufficient_statistic_is_enough_for_best_estimator}

Consider the statistical problem of estimating $\bbeta^*\in \cB$ from observations $\by \in \cY$. $\cB$ $\ell^2$-compact. If $S(\by)$ is a sufficient statistic of $\bbeta^*$, then the best estimator that achieves $\min_{\hat \bbeta}\max_{\cB} \ell(\hat\bbeta,\bbeta^*) $ is of the form $\hat \bbeta = f(S(\by))$ with some function $f$, for any loss $\ell:\cY \rightarrow [0,\infty)$. 
\end{lemma}
This Lemma is restated from Proposition 3.13 from \cite{johnstone2011gaussian}. 

\begin{corollary}[Corollary of Lemma \ref{lemma:sufficient_statistic_is_enough_for_best_estimator}]
Under the same setting of Lemma \ref{lemma:sufficient_statistic_is_enough_for_best_estimator}, $R_N(\cB)$ is achieved with the form $\hat \bbeta = f(S(\by))$. 	
\end{corollary}

\subsection{Omitted proof for utilizing source and target data jointly}

\paragraph{Sufficient statistic.}
\begin{proof}[Proof of Claim \ref{claim:sufficient_statistic}]
	Denote by $ \bar \bbeta_S :=  \hat \Sigma_S^{-1} X_S^\top \by_S/n_S \sim \cN(\bbeta^*, \frac{\sigma^2}{n_S}\hat \Sigma_S^{-1} ) $ and $\bar \bbeta_T :=\hat \Sigma_T^{-1} X_T^\top \by_T /n_T \sim \cN(\bbeta^*, \frac{\sigma^2}{n_T}\hat \Sigma_T^{-1} ). $ We use the Fisher–Neyman factorization theorem to derive the sufficient statistics.	The likelihood of observing $\bar \bbeta_S, \bar \bbeta_T$ from parameter $\bbeta^*$ is:
	\begin{align*}
	p(\bar \bbeta_S, \bar \bbeta_T; \bbeta^*) =  & c e^{- \frac{n_S}{\sigma^2}(\bar \bbeta_S - \bbeta^*) \hat \Sigma_S (\bar \bbeta_S - \bbeta^*) - \frac{n_T}{\sigma^2} (\bar \bbeta - \bbeta^*)\bar \Sigma_T (\bar \bbeta_T - \bbeta^*) } \\
	= & c g(\bbeta^*, T(\bbeta^*)) h(\bar \bbeta_S, \bar \bbeta_T),
	\end{align*}
	where $g(\bbeta^*, T(\bbeta^*)) = e^{-(\bbeta^* - \hat \bbeta_{\tSS} )^\top(\frac{n_S}{\sigma^2}\hat \Sigma_S + \frac{n_T}{\sigma^2}\hat \Sigma_T )^{-1} (\bbeta^* - \hat \bbeta_{\tSS})}$, and $c$ is some constant. Therefore it's easy to see that $T(\bbeta^*) = \hat \bbeta_{\tSS} $ is the sufficient statistic for $\bbeta^*$.  
\end{proof}	

\begin{proof}[Proof of Claim \ref{claim:linear_in_SS}]
	With similar procedure as before, and notice $\bz_S$ and $\bz_T$ are independent, we could first conclude that the optimal estimator is of the form $\hat \bbeta = A \hat \Sigma_S^{-1} X_S^\top \by_S/n_S + B \hat \Sigma_T^{-1} X_T^\top \by_T/n_T \sim \cN((A+B)\bbeta^*, \frac{\sigma^2}{n_S}A\hat \Sigma_S^{-1}A^\top + \frac{\sigma^2}{n_T} B\hat \Sigma_T^{-1} B^\top) $.
	\begin{align*}
	R_L(\cB) = & \min_{A, B} \max_{\bbeta^* \in \cB} \E_{\bz} \|\Sigma_T^{1/2}(\hat \bbeta - \bbeta^*) \|^2 \\
	= &   \min_{A, B} \max_{\bbeta^* \in \cB} \left\{ \|\Sigma^{1/2}(A+B-I)\bbeta^* \|^2 \right.  \\
	& \left. + \sigma^2 \Trace((\frac{1}{n_S}A\hat \Sigma_S^{-1}A^\top + \frac{1}{n_T} B\hat \Sigma_T^{-1} B^\top)\Sigma_T)  \right\} \\
	= & \min_{A, B} \left\{ \|\Sigma^{1/2}(A+B-I)\|^2_{op}r^2 + \sigma^2 \Trace((\frac{1}{n_S}A\hat \Sigma_S^{-1}A^\top + \frac{1}{n_T} B\hat \Sigma_T^{-1} B^\top)\Sigma_T)  \right\}
	\end{align*} 
	Take gradient w.r.t $A$ and $B$ respectively we have:
	\begin{align*}
	& \nabla_A (\|\Sigma^{1/2}(A+B-I)\|^2_{op}r^2) + \frac{\sigma^2}{n_S} \Sigma_T A \hat \Sigma_S^{-1} =0 \\
	= & \nabla_B (\|\Sigma^{1/2}(A+B-I)\|^2_{op}r^2) + \frac{\sigma^2}{n_T} \Sigma_T B \hat \Sigma_T^{-1} =0
	\end{align*}
	Notice the first terms are equivalent. Therefore $\frac{1}{n_S} A \hat \Sigma_S^{-1} = \frac{1}{n_T} B \hat \Sigma_T^{-1} $ thus the optimal $\hat\bbeta$ is of the form $C (X_S^\top \by_S + X_T^\top \by_T)$ for some matrix $C$, thus finishing the proof. 
\end{proof}

\section{Omitted proof with approximation error}
\paragraph{Unbiased estimator for $\hat\bbeta_T^*$.}
\begin{proof}[Proof of Claim \ref{claim:MVUE}]
	\begin{align*}
	\hat \bbeta_{\tLS} -\bbeta_T^* = & (X_S^\top \diag(\bw) X_S)^{-1} (X_S^\top \diag(\bw) \by) -\bbeta_T^* \\
	= &  (X_S^\top \diag(\bw) X_S)^{-1} (X_S^\top \diag(\bw) (X_S\bbeta_T^*+\ba_T+\bz ))  - \bbeta_T^* \\
	= & (X_S^\top \diag(\bw) X_S)^{-1} (X_S^\top \diag(\bw)(\ba_T+\bz))
	\end{align*}
	Notice $\E_{\bx\sim p_S}[\bx a_T(\bx) \frac{p_T(\bx)}{p_S(\bx)}] = \E_{\bx\sim p_T}[\bx a_T(\bx)]=0$. 
	This is due to the KKT condition for the minimizer of $l(\bbeta):=\E_{\bx\sim p_T}\|f^*(\bx)-\bbeta^\top\bx\|^2 $ at $\bbeta^*_T$: $\nabla_{\bbeta} f(\bbeta^*)=0\rightarrow \E_{\bx\sim p_T}[\bx(f^*-\bx^\top\bbeta^*_T)]=0$, i.e., $\E_{\bx\sim p_T}[\bx a_T(\bx)]=0$.	
	Next we have:	
	$\E_{\bx_i \sim p_S}[X_S^\top \diag(\bw) X_S ] = \E_{\bx_i\sim p_S }\sum_{i=1}^n \frac{p_T(\bx_i)}{p_S(\bx_i)} \bx_i \bx_i^\top = \E_{\bx_j\sim p_T} \sum_{j=1}^n [\bx_j\bx_j^\top]=n_S\Sigma_T$. Therefore
	\begin{equation*}
	\hat \bbeta_{\tLS} - \bbeta_T^* \rightarrow \cN(0, \frac{1}{n_S}\Sigma_T^{-1} \E_{\bx\sim p_T} [p_T(\bx)/p_S(\bx) (a_T(\bx)^2+\sigma^2)\bx\bx^\top]\Sigma_T^{-1} ).	
	\end{equation*}
\end{proof}

\begin{proof}[Proof of Claim \ref{claim:nonlinear_form}]
	Recall $X_S=[\bx_1^\top |\bx_2^\top |\cdots |\bx_n^\top]^\top \in \R^{n\times d}$, with $\bx_i,\forall i\in [n]$ drawn from $p_S$, and $\ba_T = [a_T(\bx_1),a_T(\bx_2),\cdots a_T(\bx_n)]^\top \in \R^n$, $\by = [y(\bx_1),y(\bx_2),\cdots, y(\bx_n)]^\top \in \R^n$, noise $\bz = \by - f^*(X)$. $\bw = [p_T(\bx_i)/p_S(\bx_i)]^\top$.
	
	To prove the, we only need to show the minimax linear estimator $A\by$ is achieved of the form $A_1X^\top\diag(\bw)$, i.e., the row span of $A$ is in the row span of $X^\top\diag(\bw)$. 
	\begin{align*}
	R_L(\cB) \equiv & \min_{A} \max_{\bbeta^*_T\in \cB, a_T\in \cF} \E_{\bx_i\sim p_s, \bz}[\|\Sigma_T^{1/2}(A\by - \bbeta^*_T)\|^2 ]\\
	= & \min_{A} \max_{\bbeta^*_T\in \cB, a_T\in \cF} \E \|\Sigma_T^{1/2}((AX-I)\bbeta_T^*+ A\ba_T + A z)\|^2\\
	= & \min_{A} \max_{\bbeta^*_T\in \cB, a_T\in \cF} \left\{ \|\Sigma_T^{1/2}((\E[AX]-I)\bbeta^*_T + \E[A\ba_T] ) \|_2^2  \right.\\
	& \left. + \E\|\Sigma_T^{1/2}(AX-\E[AX])\bbeta_T^*\|^2 + \E\|\Sigma_T^{1/2}(A\ba_T-\E[A\ba_T])\|^2 + 
	\E\|\Sigma_T^{1/2}A\bz \|^2\right\}
	\end{align*}	
	Write $A = A_1X^\top\diag(\bw) + A_2 W^\top$, where $X\in \R^{n\times d}$ and $W\in \R^{n\times (n-d)}$ forms the orthogonal complement for the column span of $\diag(\bw)X$. Therefore $X^\top\diag(\bw) W=0$, and $W^\top W=I_{n-d}$. Also, notice $\E_{\bx_i\sim p_S}[X^\top \diag(\bw) \ba_T] = n\E_{\bx\sim p_T}[\bx a_T(\bx)]=0$. Therefore plugging it in $R_L(\cB)$, we have:
	\begin{align*}
	R_L(\cB) = & \min_{A} \max_{\bbeta^*_T\in \cB, f^*\in\cF} \left\{ \|\Sigma_T^{1/2}((A_1\E_{p_S}[X^\top \diag(\bw) X]-I)\bbeta^*_T + A_2\E[W^\top \ba_T] ) \|_2^2  \right.\\
	& + \E\|\Sigma_T^{1/2}A_1(X^\top \diag(\bw) X-\E[X^\top \diag(\bw) X])\bbeta_T^*\|^2 \\
	& + \E\|\Sigma_T^{1/2}A_2( W^\top \ba_T-\E[W^\top \ba_T])\|^2 \\ 
	& \left. + 
	\sigma^2 \E\|\Sigma_T^{1/2}A_1 X^\top\diag(\bw)\|^2 + \sigma^2 \E\|\Sigma_T^{1/2}A_2 \|^2 \right\}\\
	= & \min_{A_1,A_2} \max_{\bbeta^*_T\in \cB, f^*\in\cF} \left\{ \|\Sigma_T^{1/2}((A_1 n_S \Sigma_T -I)\bbeta^*_T + A_2\E[W^\top \ba_T] ) \|_2^2  \right.\\
	& + \E\|\Sigma_T^{1/2}A_1(X^\top \diag(\bw) X-\Sigma_T)\bbeta_T^*\|^2 + \E\|\Sigma_T^{1/2}A_2( W^\top \ba_T-\E[W^\top \ba_T])\|^2 \\ 
	& \left. + 
	\sigma^2 \E\|\Sigma_T^{1/2}A_1 X^\top\diag(\bw)\|^2 + \sigma^2 \E\|\Sigma_T^{1/2}A_2 \|^2 \right\}
	\end{align*}
	We could view $\E[W^\top \ba_T]$ and $W^\top \ba_T - \E[W^\top \ba_T]$ separately. First notice at min-max point, if $\E[W^\top \ba_T]=0$, the minimizer $A_2$ should be 0 since it only appears in the third and last non-negative terms. If $\E[W^\top \ba_T]\neq 0,$ the cross term of the bias should be non-negative, or otherwise since both $f^*$ and $-f^*$ are in the set, $a_T,\bbeta_T^*$ could be replaced by $-a_T,-\bbeta_T^*$ and the loss increases. Clearly in this case $A_2$ should also be 0 at min-max point.   
\end{proof}

\paragraph{On estimating $p_T/p_S$.}
\begin{proof}[Proof of Proposition \ref{prop:estimate_density_ratio}]
	\begin{align*}
 \E_{x,y\sim q} (y-f(x))^2 = &  \E_{y\sim Ber(1/2)}[\E[x\sim q_{X|Y}] (y-f(x))^2 |y]\\
 = & 1/2\E_{x\sim p_T} (1-f(x))^2 + 1/2 \E_{x\sim p_S} (0- f(x))^2 \\
 = & \int_{x} p_T(\bx)(1-f(\bx))^2 + p_S(\bx)f(\bx)^2 d\bx.
	\end{align*}
For any $\bx$, the optimal value for $a:=f(\bx)$ is obtained by taking the derivative of  $p_T(\bx)(1-a)^2 + p_S(\bx)a^2$, i.e., $a=\frac{p_T(\bx)}{p_S(\bx)+p_T(\bx)}$. 
Therefore, the optimal function $f(\bx)\equiv \frac{p_T(\bx)}{p_S(\bx)+p_T(\bx)}$ for all $\bx$. 
\end{proof}

\section{Omitted Proof with Model Shift}
\label{section:model_shift_proof}

\begin{definition}[Orthosymmetry]
	A set $\Theta$ is said to be solid and orthosymmetric if $\btheta\in \Theta$ and $|\zeta_i|\leq |\theta_i|$ for all $i$ implies that $\bzeta\in \Theta$. If a solid, orthosymmetric $\Theta$ contains a point $\btau$, then it contains the entire
	hyperrectangle that $\btau$ defines: $\Theta(\btau)\equiv \{\btheta | |\theta_i|\leq \tau_i,\forall i \} \subset \Theta$. 
\end{definition}

\begin{proof}[Proof of Claim \ref{claim:beta_change_relaxed_loss}]
	First notice for any estimator $\hat\bbeta$, it all satisfies 
	\begin{equation}
	\label{eqn:tight_relaxation}
	L_{\cB,\Delta} (\hat \bbeta)  \leq r_{\cB,\Delta} (\hat \bbeta) \leq 2 L_{\cB,\Delta} (\hat \bbeta).
	\end{equation}
	The first inequality is straightforward with the same reasoning of AM-GM as the derivation of \eqref{eqn:am-gm_inequality}. As for the second inequality, we take a closer look at \eqref{eqn:am-gm_inequality}. Notice that when $\max_{\bbeta_T^*\in \cB, \bdelta\in \Delta}$ is achieved, the cross term has to be non-negative, or otherwise one could flip the sign of $\bbeta^*_T$ to make the value larger. Therefore at maximum $ \|\Sigma_T^{1/2}((A_1+A_2-I)\bbeta_T^*\|^2 + \|\Sigma_T^{1/2}A_1\bdelta \|^2\leq  \|\Sigma_T^{1/2}((A_1+A_2-I)\bbeta_T^* + \Sigma_T^{1/2}A_1\bdelta \|^2$, and notice the remaining parts are all non-negative. Therefore $r_{\cB,\Delta} (\hat \bbeta) \leq 2 L_{\cB,\Delta} (\hat \bbeta)$. 
	
	Now let $\hat \bbeta^* = \argmin_{\hat \bbeta=A_1\bar \by_S + A_2 \bar \by_S} L_{\cB,\Delta}(\hat \bbeta)$. We have:
	\begin{align*}
	R_L(\cB,\Delta) = & L_{\cB,\Delta}(\hat \bbeta^*) \overset{(a)}{\leq}  L_{\cB,\Delta}(\hat \bbeta_{\tMM})  \\
	\overset{\eqref{eqn:tight_relaxation}}{\leq} & r_{\cB,\Delta}(\hat \bbeta_{\tMM})
	\overset{(b)}{\leq} r_{\cB,\Delta}(\hat \bbeta^*) 
	\overset{\eqref{eqn:tight_relaxation}}{\leq}  2 L_{\cB,\Delta}(\hat \bbeta^*) =  2R_L(\cB,\Delta).
	\end{align*}
	The inequality (a) is by definition of $\hat \bbeta^*$ while (b) is from the definition of $\hat \bbeta_{\tMM}$.
\end{proof}

\subsection{Lower Bound with Model Shift}
In order to derive the lower bound, we abstract the problem to the following more general one:
\begin{problem}
	\label{problem:minimax_risk_with_beta_change}
	For arbitrary diagonal matrix $D\in \R^{d\times d}$, two $\ell_2$-compact, solid, orthosymmetric, and quadratically convex sets $\Theta,\Delta\subset \R^d$, let 
	\begin{equation*}
	\cP_{\Theta,\Delta,D} =   \left\{\cN \left( \left[ \left.  
	\begin{array}{c}
	D\btheta+\bdelta\\
	\btheta 
	\end{array}
	\right],
	\left[ 
	\begin{array}{c|c}
	I & 0 \\
	\hline 
	0 & I
	\end{array}
	\right] 
	\right)
	\right| \btheta\in\Theta, \bdelta \in \Delta
	\right\}
	\end{equation*}
	Let $R_L(\Theta,\Delta,D)$ and $R_N(\Theta,\Delta, D)$ be the minimax linear risk and minimax risk respectively for estimating $\btheta$ within the distribution class $\cP_{\Theta,\Delta, D}$:
	\begin{align*}
	R_L(\Theta,\Delta,D) = & \min_{\hat \btheta:\R^d\rightarrow \Theta \text{ linear}}\max_{P\in \cP_{\Theta,\Delta,D} }r_P(\hat \btheta), \\
	R_N(\Theta,\Delta,D) = & \min_{\hat \btheta:\R^d\rightarrow \Theta }\max_{P\in \cP_{\Theta,\Delta,D} }r_P(\hat \btheta).
	\end{align*}
	Here $r_P(\hat \btheta) := \E_{\bx\sim P}\|\hat \btheta(\bx) - \btheta(P)\|_2^2 $.
	 We want to derive a uniform lower bound for $R_N$ with $R_L$, i.e., $R_N\geq \mu^* R_L$, where $\mu^*$ is universal and doesn't depend on the choices of $D$, $\Theta$ or $\Delta$. 
\end{problem}
Before proving the lower bound, we establish its connection to our considered problem: 
\begin{remark}
	\label{remark:equivalence_of_our_problem}
	Suppose $\Sigma_S=U\diag(\bs)U^\top$ and $\Sigma_T=U\diag(\bt)U^\top$ share the same eigenspace. Recall our samples $\ba\sim \cN(\Sigma_S^{1/2}(\bbeta^*_T+\bdelta),\sigma^2 I), \bb\sim \cN(\Sigma_T^{1/2}\bbeta^*_T,\sigma^2 I)$. Our goal to uniformly lower bound $R_N(r,\gamma)$ by $R_L(r,\gamma)$ is essentially Problem \ref{problem:minimax_risk_with_beta_change}, where
	\begin{align*}
	R_L(r,\gamma)	:= & \min_{\hat\bbeta \text{ linear}}\max_{\|\bbeta_T^*\|\leq r,\|\bdelta\|\leq \gamma}\E\|\Sigma_T^{1/2}(\hat\bbeta(\ba,\bb)-\bbeta^*)\|^2,	\\
	R_N(r,\gamma)	:= & \min_{\hat\bbeta}\max_{\|\bbeta_T^*\|\leq r,\|\bdelta\|\leq \gamma}\E\|\Sigma_T^{1/2}(\hat\bbeta(\ba,\bb)-\bbeta^*)\|^2.
	\end{align*}
\end{remark}
\begin{proof}[Proof of Remark \ref{remark:equivalence_of_our_problem}]
Our target considers samples drawn from distributions $\bx\sim \cN(\Sigma_S^{1/2}(\bbeta^*_T+\bdelta),\sigma^2 I), \by\sim \cN(\Sigma_T^{1/2}\bbeta^*_T,\sigma^2 I)$. 
\begin{align*}
& \left[
\begin{array}{c}
\ba \\
\bb 
\end{array} 
\right] \sim \cN \left( \left[ 
\begin{array}{c}
U\diag(\bs^{1/2})U^\top (\bbeta_T^*+\bdelta)\\
U\diag(\bt^{1/2})U^\top \bbeta_T^*
\end{array}
\right],
\left[ 
\begin{array}{c|c}
\sigma^2 I & 0 \\
\hline 
0 & \sigma^2 I
\end{array}
\right] 
\right), 
\btheta\in\Theta, \bdelta \in \Delta
\\
\Longleftrightarrow & \left[
	\begin{array}{c}
	U^\top\ba/\sigma \\
	U^\top\bb/\sigma  
	\end{array} 
	\right] \sim \cN \left( \left[ 
	\begin{array}{c}
	\diag(\bs^{1/2})U^\top (\bbeta_T^*+\bdelta)\\
	\diag(\bt^{1/2})U^\top \bbeta_T^*
	\end{array}
	\right],
	\left[ 
	\begin{array}{c|c}
	I & 0 \\
	\hline 
	0 & I
	\end{array}
	\right] 
	\right), 
	 \|\bbeta_T^*\|\leq r, \|\bdelta\| \in \gamma 
\end{align*} 
Let $\bar\ba=U^\top\ba/\sigma,\bar\bb=U^\top\bb/\sigma, \Theta = \{\btheta| \|\diag(\bt^{-1/2})\btheta\|\leq r \}$, $\Delta=\{ \|\diag(\bs^{-1/2})\bdelta\|\leq \gamma \}$. $\bar\btheta = U^\top \Sigma_T^{1/2}\bbeta_T^*, \bdelta = U^\top \Sigma_S^{1/2}\bdelta$, and $D=\diag(\bs^{1/2}\bt^{-1/2})$. We get:
\begin{align*} 
& \left[
\begin{array}{c} 
U^\top\ba/\sigma \\
U^\top\bb/\sigma  
\end{array} 
\right] \sim \cN \left( \left[ 
\begin{array}{c}
\diag(\bs^{1/2})U^\top (\bbeta_T^*+\bdelta)\\
\diag(\bt^{1/2})U^\top \bbeta_T^*
\end{array}
\right],
\left[ 
\begin{array}{c|c}
I & 0 \\
\hline 
0 & I
\end{array}
\right] 
\right), 
\|\bbeta_T\|\leq r, \|\bdelta\| \in \gamma \\
	\Longleftrightarrow &  \left[
	\begin{array}{c}
\bar\ba \\
	\bar\bb 
	\end{array} 
	\right] \sim P_{\btheta,\bdelta,D}:= \cN \left( \left[ 	\begin{array}{c}
	D \bar\btheta + \bar\bdelta\\
	\bar\btheta 
	\end{array}
	\right],
	\left[ 
	\begin{array}{c|c}
	I & 0 \\
	\hline 
	0 & I
	\end{array}
	\right] 
	\right), \bar\btheta \in\Theta, \bar\bdelta \in \Delta.
	\end{align*}
Let $ \cP_{\Theta, \Delta,D} := \left\{ \left. P_{\bar\btheta,\bar\bdelta,D} 
\right| \bar\btheta \in\Theta, \bar\bdelta \in \Delta
\right\}$. Since $U$ is an invertible matrices, observing $U^\top \ba/\sigma, U^\top \bb/\sigma$ instead of $\ba, \bb$ has no affect on the performance of the best estimator. Also $\Theta,\Delta$ are axis-aligned ellipsoid and thus satisfy orthosymmetry. Therefore our problem is essentially reduced to Problem \ref{problem:minimax_risk_with_beta_change}. 
\end{proof}

\begin{lemma}
	\label{lemma:relate_to_worst_hyperrectangle} 
Let	$\Theta(\btau) = \{ \btheta | \theta_i \leq \tau_i,\forall i, \btheta \in \Theta\}$ and similarly for $\Delta(\bzeta) = \{\bdelta | \bdelta_i \leq \zeta_i, \bdelta \in \Delta \}$, $D$ is some diagonal matrix.
$$R_L(\Theta,\Delta,D) = \sup_{\btau\in \Theta, \bzeta \in \Delta }R_L(\Theta(\btau), \Delta(\bzeta),D), \text{ and} $$
$$R_N(\Theta,\Delta,D) \geq \sup_{\btau\in \Theta, \bzeta \in \Delta }R_N(\Theta(\btau), \Delta(\bzeta),D).  $$
\end{lemma}

Write samples drawn from some $P_{\btheta,\bdelta,D} \in \cP_{\Theta,\Delta,D}$ as $ (\bx,\by): \bx\sim \cN(D\btheta+\bdelta, I), \by\sim \cN(\btheta,I)$. 
\begin{lemma}
	\label{lemma:best_linear_estimator_diagonal}
The minimax linear estimator $\hat\btheta:(\bx,\by)\rightarrow A\bx+B\by$ has the form $\hat\btheta_{\ba,\bb}(\bx,\by)=\sum_ia_ix_i + \sum_ib_i y_i $ for some $\ba,\bb\in \R^d$. Namely,
$$ R_L(\Theta,\Delta,D) = \inf_{\hat\btheta_{\ba,\bb}} \max_{P\in \cP_{\Theta,\Delta,D}} r_P(\hat\btheta_{\ba,\bb}).  $$
\end{lemma}
\begin{proof}
According to the proof of Proposition \ref{proposition:decompose}.a, by discarding off-diagonal terms, the maximum risk of any linear estimator $\hat\btheta_{A,B} $ over any hyperrectangles $\Theta(\btau),\Delta(\bzeta)$ is reduced. 
$$\max_{\btheta \in \Theta(\btau), \bdelta\in\Delta(\bzeta) } r_{P_{\btheta,\bdelta,D}}(\hat\btheta_{A,B}) \geq  \max_{\btheta \in \Theta(\btau), \bdelta\in\Delta(\bzeta) }r_{P_{\btheta,\bdelta,D}}(\hat\btheta_{\diag(A),\diag(B)}). $$	
Further we have:
\begin{align*}
\min_{A,B} \max_{\btheta\in \Theta,\bdelta\in\Delta}r_{P_{\btheta,\bdelta,D}}(\hat\btheta_{A,B}) \geq & \min_{A,B} \max_{\btau\in \Theta,\bzeta\in \Delta} \max_{\btheta \in \Theta(\btau), \bdelta\in\Delta(\bzeta) } r_{P_{\btheta,\bdelta,D}}(\hat\btheta_{\diag(A),\diag(B)}) \\
	= & \min_{\ba,\bb} \max_{\btheta\in \Theta,\bzeta\in \Delta} r_{P_{\btheta,\bdelta,D}}(\hat\btheta_{\ba,\bb})\\
	\geq & \min_{C} \max_{\btheta\in\Theta,\bdelta\in\Delta} r_{P_{\btheta,\bdelta,D}}(\hat\btheta_{A,B}).
\end{align*}
Therefore all four terms have to be equal, thus finishing the proof.  
\end{proof}

Notice $\Theta(\btau)$ and $ \Delta(\bzeta) $ are hyperrectangles in $\R^d$. Therefore we could decompose the problem to some 2-d problems: 

\begin{proposition}
		\label{proposition:decompose}
	Under the same setting as Problem \ref{problem:minimax_risk_with_beta_change}, 
	$$a).~\  R_L(\Theta(\btau), \Delta(\bzeta), D) = \sum_i R_L(\tau_i, \zeta_i, D_{ii}).$$ 
	If $\hat\btheta_{A,B}(\bx,\by)=A\bx+B\by$ is minimax linear estimator over $P_{\Theta(\btau), \Delta(\bzeta), D},$ then necessarily $A,B$ must be diagonal. 
		$$b).~\ R_N(\Theta(\btau), \Delta(\bzeta), D) = \sum_i R_N(\tau_i, \zeta_i, D_{ii}).$$ 
\end{proposition}
\begin{proof}[Proof of Proposition \ref{proposition:decompose}.a ]
First review our notation: 
\begin{align*}
r_{P_{\btheta, \bdelta, D}}(\hat \btheta_{A,B}) = & \E_{(\bx,\by)\sim P_{\btheta, \bdelta, D}}\|\hat \btheta_{A,B}(\bx,\by)-\btheta\|^2\\
= & \E_{\bx\sim\cN(D\btheta+\bdelta,I),\by\sim \cN(\btheta, I)} \|A\bx+B\by-\btheta\|^2\\
= & \|A(D\btheta+\bdelta) + B\btheta -\btheta \|^2 + \Trace(AA^\top) + \Trace(BB^\top)\\
= & \|(AD+B-I)\btheta + A\bdelta \|^2 + \Trace(AA^\top) + \Trace(BB^\top). 
\end{align*}
Our objective is
$$R_L( \Theta(\btau),\Delta(\bzeta),D  ) := \min_{A,B}\max_{\btheta\in \Theta(\btau),\bdelta\in \Delta(\bzeta)}r_{P_{\btheta, \bdelta, D}}(\hat \btheta_{A,B})  $$
We will show that restricting $A,B$ to be diagonal will not include the RHS value. 	
	
For any $\bar\btau\in \Theta(\btau), \bar\bzeta\in\Delta(\bzeta)$, let set $ V(\bar\btau,\bar\bzeta)=\{(\btheta,\bdelta)|(\theta_i,\delta_i)\in \{(\bar\tau_i,\bar\zeta_i),(-\bar\tau_i,-\bar\zeta_i) \}  \}$ be the subset of vertices of $\Theta(\bar\btau)\times \Delta(\bar\bzeta) $. Let $\pi(\bar\btau,\bar\bzeta)$ be uniform distribution on this finite set. Due to the symmetry of this distribution, we have 
$$\E_{\pi(\bar\btau,\bar\bzeta)}\theta_i=0,i\in [d],  $$
$$\E_{\pi(\bar\btau,\bar\bzeta)}\delta_i=0,i\in [d],  $$
$$\E_{\pi(\bar\btau,\bar\bzeta)}\theta_i\theta_j=\onebb_{i=j}\bar\tau_i^2 ,i\in [d],  $$
$$\E_{\pi(\bar\btau,\bar\bzeta)}\delta_i\delta_j=\onebb_{i=j}\bar\zeta_i^2 ,i\in [d],  $$
$$\E_{\pi(\bar\btau,\bar\bzeta)}\theta_i\delta_j=\onebb_{i=j}\bar\tau_i\bar\zeta_i ,i\in [d].  $$
We utilize the distribution to find the explicit value of the maximum (in fact the maximum will only be obtained inside the vertices set $V(\bar\btau,\bar\bzeta)$ ):
\begin{align*}
& \max_{(\btheta,\bdelta)\in V(\bar\btau,\bar\bzeta)} r_{P_{\btheta, \bdelta, D}}(\hat \btheta_{A,B}) \geq  \E_{\pi(\bar\btau,\bar\bzeta)} r_{P_{\btheta, \bdelta, D}}(\hat \btheta_{A,B}) \\
= & \E_{\pi(\bar\btau,\bar\bzeta)} \|(AD+B-I)\btheta + A\bdelta \|^2 + \Trace(AA^\top) + \Trace(BB^\top) \\
= & \Trace((AD+B-I)\E[\btheta\btheta^\top](AD+B-I)^\top) + \Trace(A\E[\bdelta\bdelta^\top]A^\top ) +\\
& 2\Trace((AD+B-I)\E[\btheta\bdelta^\top]A^\top) + \Trace(AA^\top)+\Trace(BB^\top)\\
=& \Trace((AD+B-I)^\top(AD+B-I)\diag(\bar\btau^2))+\Trace(A^\top A\diag(\bar\bzeta^2)) \\ 
& + \Trace((AD+B-I)^\top A\diag(\bar\btau\bar\bzeta)) +\Trace(AA^\top)+\Trace(BB^\top)\\
= & \sum_i \|(AD+B-I)_{:,i}\bar\tau_i+A_{:,i}\bar\zeta_i \|^2 + \Trace(AA^\top)+\Trace(BB^\top)\\
\geq & \sum_i ((A_{ii}D_{ii}+B_{ii}-1)\bar\tau_i+A_{ii}\bar\zeta_i)^2+A_{ii}^2+B_{ii}^2\\
= & \|(\diag(A)D+\diag(B)-I) \btheta + \diag(A)\bdelta \|^2 +\Trace(\diag(A)^2)+\Trace(\diag(B)^2),\tag{$\forall (\btheta,\bdelta)\in V(\bar\btau,\bar\bzeta)$}\\
= & \max_{ V(\bar\btau,\bar\bzeta)}\|(\diag(A)D+\diag(B)-I) \btheta + \diag(A)\bdelta \|^2 +\Trace(\diag(A)^2)+\Trace(\diag(B)^2)
\end{align*}

Therefore we have:
\begin{align*}
R_L( \Theta(\btau),\Delta(\bzeta),D  ) := & \min_{A,B}\max_{\btheta\in \Theta(\btau),\bdelta\in \Delta(\bzeta)}r_{P_{\btheta, \bdelta, D}}(\hat \btheta_{A,B}) \\
= & \min_{A,B}\max_{\bar\btau\in \Theta(\btau),\bar\bzeta\in \Delta(\bzeta)}\max_{\btheta\in V(\bar\btau,\bar\bzeta)} r_{P_{\btheta, \bdelta, D}}(\hat \btheta_{A,B})\\
\geq & \min_{A,B}\max_{\bar\btau\in \Theta(\btau),\bar\bzeta\in \Delta(\bzeta)}\max_{(\btheta,\bdelta)\in V(\bar\btau,\bar\bzeta)} r_{P_{\btheta, \bdelta, D}}(\hat \btheta_{\diag(A),\diag(B)})\\
= & \min_{\ba\in\R^d,\bb\in\R^d}\max_{\btheta\in \Theta(\btau),\bdelta\in\Delta(\bzeta)} r_{P_{\btheta, \bdelta, D}}(\hat \btheta_{\ba,\bb}).
\end{align*} 
Next, since the optimal solution on the minimizer is always obtained by diagonal $A,B$, it becomes straightforward that each axis could be viewed in separation, thus finishing the proof for part a. 

The nonlinear part is a straightforward extension of Proposition 4.16 from \cite{johnstone2011gaussian}. 
	
\end{proof}

\begin{theorem}[Restated Le Cam Two Point Theorem \cite{wainwright2019high}]
	Let $\cP$ be a family of distribution, and
	$\theta: \cP\rightarrow \Theta$ is some associated parameter.  Let $\rho: \Theta \times \Theta \rightarrow \R^+$ be some metric defined on $\Theta$ and $\Phi:\R_+\rightarrow \R_+$ is a monotone non-decreasing function with $\Phi(0)=0$. For any $\alpha \in (0,1)$, 
	\begin{align*}
	\inf_{\hat \theta} \sup _{P\in \cP}[\Phi(\rho(\hat \theta, \theta(P))) ]  \geq & \max_{P_1,P_2\in \cP} \frac{1}{2}\Phi(\frac{1}{2}\rho(\theta(P_1),\theta(P_2))  ) (1-\alpha),\\
	&\text{ s.t. } \|P_1^n - P_2^n \|_{TV}\leq \alpha. 
	\end{align*}
\end{theorem} 
	
	\begin{lemma}
		\label{lemma:1d_bound_with_beta_change}
		Consider a class of distribution $\cP_{\tau,\zeta, s} = \{ P_{\theta,\delta,s}| P_{\theta,\delta,s}:=\cN( [s\theta+\delta, \theta]^\top, I_{2}), |\theta|\leq \tau, |\delta|\leq \zeta \}$. 
		Define 
		\begin{align*}
		R_L(\tau,\zeta, s) = &\min_{\hat \theta\text{ linear}}\max_{|\theta|\leq \tau, |\delta|\leq \zeta} \E_{\bx\sim P_{\theta,\delta,s}} (\hat\theta(\bx) - \theta )^2,\\
		\text{and }	R_N(\tau,\zeta, s) = & \min_{\hat \theta}\max_{|\theta|\leq \tau, |\delta|\leq \zeta}\E_{\bx\sim P_{\theta,\delta,s}} (\hat\theta(\bx) - \theta )^2	
		\end{align*}
		We have
		$$R_L(\tau,\zeta, s)\leq 27/2 R_N(\tau,\zeta, s), \forall \zeta, s>0,\tau>0.$$
	\end{lemma}
\begin{proof}[Proof of Lemma \ref{lemma:1d_bound_with_beta_change}]
We first calculate an upper bound of $R_L$ and connect it to a lower bound of $R_N$. 
	\begin{align*}
	R_L(\tau,\zeta,s) =& \min_{a,b}\max_{|\theta|\leq \tau,|\delta|\leq \zeta} [(as+b-1)\theta+a\delta]^2+a^2+b^2\\
	= & \min_{a,b} (|as+b-1|\tau + |a|\zeta)^2 + a^2+b^2\\
	\leq & \min_{a,b} 2(as+b-1)^2\tau^2 + 2a^2\zeta^2  + a^2+b^2. 
	\end{align*}	
By some detailed calculations, we get the RHS is equal to:
\begin{align*}
& \frac{2\tau^2(2\zeta^2+1) }{2\tau^2 (s^2+2\zeta^2+1) + 2\zeta^2 +1 }\\
\leq & \min\{1, 2\tau^2, \frac{1+4\zeta^2}{s^2+1} \}.
\end{align*}	
For simplify this form, we could see that

Next, we use Le cam two point theorem to lower bound $R_N(\tau,\zeta,s)$ where the metric $\rho$ is Euclidean distance and $\Phi$ is squared function. Therefore

\begin{align*}
R_N(\tau,\zeta,s) \geq & \max_{|\theta_i|\leq \tau,|\delta_i|\leq \zeta, i\in\{1,2\}}\frac{1}{2}(\frac{1}{2}(\theta_1-\theta_2) )^2(1-\alpha)  \\
 \text{s.t. } & \| \cN([s\theta_1+\delta_1,\theta_1]^\top, I_2),\cN([s\theta_2+\delta_2,\theta_2]^\top, I_2) \|_{TV}\leq \alpha. 
\end{align*}
Since the total variation distance is related to Kullback-Leibler divergence by Pinsker's inequality: $\|\cdot,\cdot\|_{TV}\leq \sqrt{\frac{1}{2}D_{KL}(\cdot \| \cdot)}$, it's sufficient to replace the constraint as:
$$ D_{KL}\left(\cN([s\theta_1+\delta_1,\theta_1]^\top, I_2) \left\| \cN([s\theta_2+\delta_2, \theta_2]^\top, I_2) \right.\right) \leq 2\alpha^2.   $$
\begin{align*}
& \max_{|\theta_i|\leq \tau,|\delta_i|\leq \zeta, i\in\{1,2\}} \frac{1}{8} (\theta_1-\theta_2)^2(1-\alpha)\\
&\text{s.t. } (s\theta_1+\delta_1 - (s\theta_2+\delta_2))^2 + (\theta_1-\theta_2)^2 \leq 2\alpha^2 \\
\Leftrightarrow & \max_{|c|\leq 2\tau, |d|\leq 2\zeta} \frac{c^2}{8} (1-\alpha)\\
&\text{s.t. } (sc+d)^2 + c^2 \leq 2\alpha^2. \\
\end{align*}
Recall $R_L\leq \min\{1,2\tau^2, \frac{1+4\zeta}{s^2+1} \}$. 

We first note that $c^2\leq 4\tau^2$ and setting $\alpha=0$ we have $R_N\geq \tau^2/2 \geq 1/4 R_L$. For In the following we look at other cases when the bound for $c^2$ is smaller. 

When $2\zeta\geq sc,$ will set $d=-sc$ and $c^2=2\alpha^2$. Let $\alpha=2/3$ for large $\tau$ we get : $c^2(1-\alpha)/8 = 2/27 \geq 2/27 R_L$. 

When $2\zeta \leq sc $ we set $d=-2\zeta$ and require $(sc-2\zeta)^2+c^2\leq 2\alpha^2$. 
We have 
$(sc-2\zeta)^2+c^2 = s^2c^2+4\zeta^2 -4\zeta sc + c^2 \leq s^2c^2+4\zeta^2 - 8\zeta^2 + c^2 = (s^2+1) c^2 - 4\zeta^2$.  Therefore as we set $c^2 = \frac{2\alpha^2+4\zeta^2}{s^2+1}$, the original inequality is satisfied. Again by setting $\alpha=2/3$   
we have $  c^2\geq 8/9 \frac{1+4\zeta^2}{s^2+1} \geq 8/9 R_L $. Therefore in this case $R_N\geq \frac{2}{27}R_L $.

\end{proof}


\section{Discussions on Random Design under Covariate Shift.}
\label{appendix:random_design}
In the main text, we present the results where we consider $X_S$ as fixed and $\Sigma_T$ to be known. In this section, we view both source and target input data as random, and generalize the results of Section \ref{sec:cov_shift} while training is on finite observations and testing is on the (worst case) population loss, under some light-tail properties of the input data samples.

\begin{proof}[Proof of Theorem \ref{thm:random_design_target}]
The proof relies on the two technical claims \ref{claim:concentration_covariance}, \ref{claim:estimation_error_worst_case}. 

Let $\hat \bbeta_R$ be the optimal linear estimator on $L_{\cB}$, i.e., $L_{\cB}(\hat \bbeta_R)=\min_{\bbeta\text{ linear in }\by_S}L_{\cB}(\bbeta)=R_L(\cB)$.
\begin{align*}
& L_{\cB}(\hat \bbeta) \leq (1+O(\sqrt{\frac{\rho^4(d+\log(1/\delta))}{n}})) \hat L_{\cB}(\hat \bbeta) \tag{Claim \ref{claim:estimation_error_worst_case}} \\
\leq & (1+O(\sqrt{\frac{\rho^4(d+\log(1/\delta))}{n}})) \hat L_{\cB}(\hat \bbeta_R) \tag{from definition of $\hat\bbeta$} \\
\leq & (1+O(\sqrt{\frac{\rho^4(d+\log(1/\delta))}{n}}))^2 L_{\cB}(\hat \bbeta_R) \tag{Claim \ref{claim:estimation_error_worst_case}}\\
\leq & (1+O(\sqrt{\frac{\rho^4(d+\log(1/\delta))}{n}})) L_{\cB}(\hat \bbeta) = (1+O(\sqrt{\frac{\rho^4(d+\log(1/\delta))}{n}}))  R_L(\cB) \tag{from $\frac{\rho^4(d+\log(1/\delta))}{n}\ll 1$, and definition of $\hat\bbeta_R$}.
\end{align*}
From Theorem \ref{thm:1.25minimax_risk} we know $R_L(\cB)\leq 1.25R_N(\cB)$ when $\Sigma_T$ is rank-1 matrix or commute with $\hat\Sigma_S$ which further finishes the whole proof. 
\end{proof}

\begin{claim}[Restated Claim A.6 from \cite{du2020few}] \label{claim:concentration_covariance}
Fix a failure probability $\delta\in(0,1)$, and assume $n\gg \rho^4(d+\log(1/\delta))$
\footnote{When this is not satisfied the result is still satisfied by replacing $O(\sqrt{\frac{\rho^4(d+\log(1/\delta))}{n}})$ with $O(\max\{\sqrt{\frac{\rho^4(d+\log(1/\delta))}{n}},\frac{\rho^2(d+\log(1/\delta))}{n}\})$. For cleaner presentation, we assume $n$ is large enough and simplify the results. }.
	Then with probability at least $1-\frac{\delta}{10}$ over the inputs $\bx_1, \ldots, \bx_{n}$, if $\bx_i\sim p$ and $p$ is a $\rho^2$-subgaussian distribution, we have
	\begin{equation} \label{eqn:concentration_source_covariance}
	(1-O(\sqrt{\frac{\rho^4(d+\log(1/\delta))}{n}})) \Sigma \preceq \frac{1}{n} X^\top X \preceq (1+O(\sqrt{\frac{\rho^4(d+\log(1/\delta))}{n}})) \Sigma,
	\end{equation}
where $\Sigma=\E_{\bx\sim p}[\bx\bx^\top] $.
\end{claim}

With the help of Claim \ref{claim:concentration_covariance} we directly get:
\begin{claim}
	\label{claim:estimation_error_worst_case}
Fix a failure probability $\delta\in(0,1)$, and assume $n_U\gg \rho^4(d+\log(1/\delta))$, $X_T = [\bx_1,\cdots,\bx_{n_U}]^\top\in \R^{n_U\times d} $ satisfies $\bx_i\sim p_T$ where $p_T$ is $\rho^2$-subgaussian. We have for any estimator $\bbeta$: 
	$$(1-O(\sqrt{\frac{\rho^4(d+\log(1/\delta))}{n_U}})) L_{\cB}(\bbeta) \leq  \hat L_{\cB}(\bbeta)\leq  (1+O(\sqrt{\frac{\rho^4(d+\log(1/\delta))}{n_U}})) L_{\cB}(\bbeta),$$ with high probability $1-\delta/10$ over the random samples $X_T$.
\end{claim}

\begin{proof}[Proof of Claim \ref{claim:estimation_error_worst_case}]
	Recall
	$$\hat L_\cB(\hat \bbeta) = \max_{\bbeta^*\in \cB} \E_{\by_S}\frac{1}{n_U} \|X_T(\hat\bbeta(\by_S)-\bbeta^*)\|^2,  $$
	$$L_{\cB}(\hat\bbeta)=\max_{\bbeta^*\in \cB} \E_{\by_S} \|\Sigma_T^{1/2}(\hat\bbeta(\by_S)-\bbeta^*)\|^2. $$	
Therefore for any estimator $\hat\bbeta,$ it satisfies 
\begin{align*}
&  L_\cB(\hat \bbeta) - \hat L_\cB(\hat \bbeta) \\
 = & (\hat\bbeta(\by_S)-\bbeta^*)^\top (\Sigma_S- \hat \Sigma_S)(\hat\bbeta(\by_S)-\bbeta^*)\\
 \lesssim &  O(\sqrt{\frac{\rho^4(d+\log(1/\delta))}{n_U}}) (\hat\bbeta(\by_S)-\bbeta^*)^\top \Sigma_S(\hat\bbeta(\by_S)-\bbeta^*)\\
 = & O(\sqrt{\frac{\rho^4(d+\log(1/\delta))}{n_U}}) L_\cB(\hat \bbeta),
\end{align*}
which finishes the proof. 
\end{proof}

\subsection{Random design on source domain.}
In the main text or the subsection above, the worst case excess risk is upper bounded by $1.25R_N$, which is achieved by best estimator that is using the same set of training data ($X_S,\by_S$). Here we would like to take into consideration the randomness of $X_S$ and compare the worst case excess risk using our estimator with a stronger notion of linear estimator.

For this purpose, we consider estimators that are linear functionals of $\by_{R}: = \Sigma_S^{1/2}\bbeta^*+\bz \in\R^d, \bz\sim \mathcal{N}(0,\sigma^2/n_S I_d)$ (this $\sigma^2/n_S$ is the correct scaling since $X_S^\top X_S/n_S$ is comparable to $\Sigma_S$). We consider the minimax linear estimator with $\by_R$ and with access to $\Sigma_S$, and we compare our estimator against this oracle linear estimator. This estimator is not computable in practice since $\Sigma_S$ must be estimated, but we will show that our estimator is within an absolute multiplicative constant in minimax risk of the oracle linear estimator. 

To recap the notations and setup, let 
\begin{align*}
\hat L_\cB(\hat\bbeta) := & \max_{\bbeta^*}\E_{\by_S} \frac{1}{n_U}\|X_T(\hat\beta(\by_S)-\bbeta^* )\|^2,\\
L_\cB(\hat\bbeta) := & \max_{\bbeta^*}\E_{\by_S}\E_{\bx\sim p_T}\|\bx^\top(\hat\beta(\by_S)-\bbeta^* )\|^2,\\
L_{\cB,R}(\hat\bbeta) := & \max_{\bbeta^*}\E_{\by_R}\E_{\bx\sim p_T}\|\bx^\top(\hat\beta(\by_R)-\bbeta^* )\|^2.
\end{align*}
Our target is to find the best linear estimator using $\hat L_\cB(\hat\bbeta)$ (trained with $X_T$) and prove its performance on the population (worst-case) excess risk $L_\cB(\hat\bbeta)$ is no much worse compared to the minimax linear risk trained on $\by_R$ and $\Sigma_S$. 

\begin{theorem}
	\label{thm:random_design_source_target}
	Fix a failure probability $\delta\in(0,1)$. Suppose both target and source distributions $p_S$ and $p_T$ are $\rho^2$-subgaussian, and the sample sizes in source domain and target domain satisfies $n_S,n_U\gg \rho^4(d+\log\frac{1}{\delta})$. Let $\hat C$ be the solution for Eqn.\eqref{eqn:cov_shift_alg_finite_n_U}, and set $\hat\bbeta(\by_S) \leftarrow \hat C \hat \Sigma_S^{-1}X_S^\top\by_S$. Then with probability at least $1-\delta$ over all the unlabeled samples from target domain and all the labeled samples $X_S$ from source domain, our estimator $\hat\bbeta(\by_R)$ yields the worst case expected excess risk that satisfies:
	$$  L_{\cB} (\hat\bbeta) \leq \left(1+ O(\sqrt{\frac{\rho^4(d+\log(1/\delta))}{n_U}})+ O(\sqrt{\frac{\rho^4(d+\log(1/\delta))}{n_T}})\right) \min_{\bbeta\text{ linear in }\by_R} L_{R,\cB}(\bbeta). $$
\end{theorem}

\begin{proof}[Proof of Theorem \ref{thm:random_design_source_target}]
For each matrix $C\in \R^{d\times d}$, we first conduct bias-variance decomposition and rewrite each worst-case risk with linear estimator in terms of a matrix $C$. When $\hat\bbeta(\by_S)= C \hat \Sigma_S^{-1}X_S^\top\by_S$, we have:
\begin{align*}
\hat L_{\cB}(\hat\bbeta) = & \|\hat\Sigma^{1/2}_T(C-I)\|_{op}^2r^2 + \frac{\sigma^2}{n}\Trace(\hat\Sigma_TC\hat\Sigma_S^{-1}C^\top)=: \hat l(C),\\
L_{\cB}(\hat\bbeta) = & \|\Sigma_T^{1/2}(C-I)\|_{op}^2r^2 + \frac{\sigma^2}{n}\Trace(\Sigma_TC\hat\Sigma_S^{-1}C^\top)=: l(C),
\end{align*}
Similarly, when $\hat\bbeta_R = C\Sigma_S^{-1/2}\by_R$, we have:
\begin{align*}
L_{R,\cB}(\hat\bbeta) = & \|\Sigma_T^{1/2}(C-I)\|_{op}^2r^2 + \frac{\sigma^2}{n}\Trace(\Sigma_TC\Sigma_S^{-1}C^\top)=: l_R(C).
\end{align*}

\begin{claim}
	\label{claim:relative_bound_3_terms}
	Fix a failure probability $\delta\in(0,1)$, and assume $n_U,n_S\gg \rho^4(d+\log(1/\delta))$, $X_S\in\R^{n_S\times d}, X_T\in\R^{n_U\times d}$ are respectively from $p_S$ $p_T$ which are both $\rho^2$-subgaussian. We have for any matrix $C\in\R^{d\times d}$: 
	$$(1-O(\sqrt{\frac{\rho^4(d+\log(1/\delta))}{n_U}})) \hat l(C) \leq  l(C) \leq  (1+O(\sqrt{\frac{\rho^4(d+\log(1/\delta))}{n_U}})) \hat l(C),$$ with high probability $1-\delta/10$ over the random samples $X_T$.
	$$(1-O(\sqrt{\frac{\rho^4(d+\log(1/\delta))}{n_S}})) l(C) \leq  l_{R}(C) \leq  (1+O(\sqrt{\frac{\rho^4(d+\log(1/\delta))}{n_S}})) l(C),$$ with high probability $1-\delta/10$ over the random samples $X_S$.	
\end{claim}
\begin{proof}[Proof of Claim \ref{claim:relative_bound_3_terms}]
We omit the proof of the first inequality since it's exactly the same as proof of Claim \ref{claim:estimation_error_worst_case}. 

For the second line, we have:
\begin{align*}
l_R(C)-l(C) = & \frac{\sigma^2}{n_S}\Trace(\Sigma_TC(\Sigma_S^{-1}-\hat\Sigma_S^{-1})C^\top)\\
\leq & O(\sqrt{\frac{\rho^4(d+\log(1/\delta))}{n_S}}) \frac{\sigma^2}{n_S}\Trace(\Sigma_TC\hat\Sigma_S^{-1}C^\top)\\
\leq & O(\sqrt{\frac{\rho^4(d+\log(1/\delta))}{n_S}}) l(C). 
\end{align*} 
Therefore we prove the RHS of the second inequality. The LHS follows with the same proof techniques. 
\end{proof}
Now let $\hat C$ be the minimizer for $\hat l(C)$, and $C_R$ be the minimizer for $l_R(C)$. 
\begin{align*}
l(\hat C) \leq & (1+ O(\sqrt{\frac{\rho^4(d+\log(1/\delta))}{n_U}})) \hat l(\hat C) \tag{w.p. $1-\delta/10$; due to Claim \ref{claim:relative_bound_3_terms} }\\
\leq & (1+ O(\sqrt{\frac{\rho^4(d+\log(1/\delta))}{n_U}})) \hat l(C_R) \tag{Due to the definition of $\hat C$ }\\
\leq & (1+ O(\sqrt{\frac{\rho^4(d+\log(1/\delta))}{n_U}}))^2 l( C_R) \tag{w.p. $1-\delta/5$; due to Claim \ref{claim:relative_bound_3_terms}} \\
= & (1+ O(\sqrt{\frac{\rho^4(d+\log(1/\delta))}{n_U}})) l( C_R) \tag{since $n_U$ is large enough}\\
\leq & (1+ O(\sqrt{\frac{\rho^4(d+\log(1/\delta))}{n_U}}))(1+ O(\sqrt{\frac{\rho^4(d+\log(1/\delta))}{n_T}})) l_R(C_R) \tag{w.p. $1-3\delta/10$; due to Claim \ref{claim:relative_bound_3_terms}}\\
= & \left(1+ O(\sqrt{\frac{\rho^4(d+\log(1/\delta))}{n_U}})+ O(\sqrt{\frac{\rho^4(d+\log(1/\delta))}{n_T}})\right) \min_C l_R(C).
\end{align*}
This finishes the proof. 
\end{proof}

\section{More empirical results}
We include some more empirical studies. In the main text our results have small noise. Here we show some more results with larger noise, and also the case with varied eigenspace. For the following results, we use $\sigma=10$ and $r=0.2\sqrt{d}$. Other meta data remains the same as presented in the main text. 
\begin{figure*}[htbp]
\begin{tabular}{ccc}
	\hspace{-0.45cm}
\includegraphics[width=0.37\linewidth]{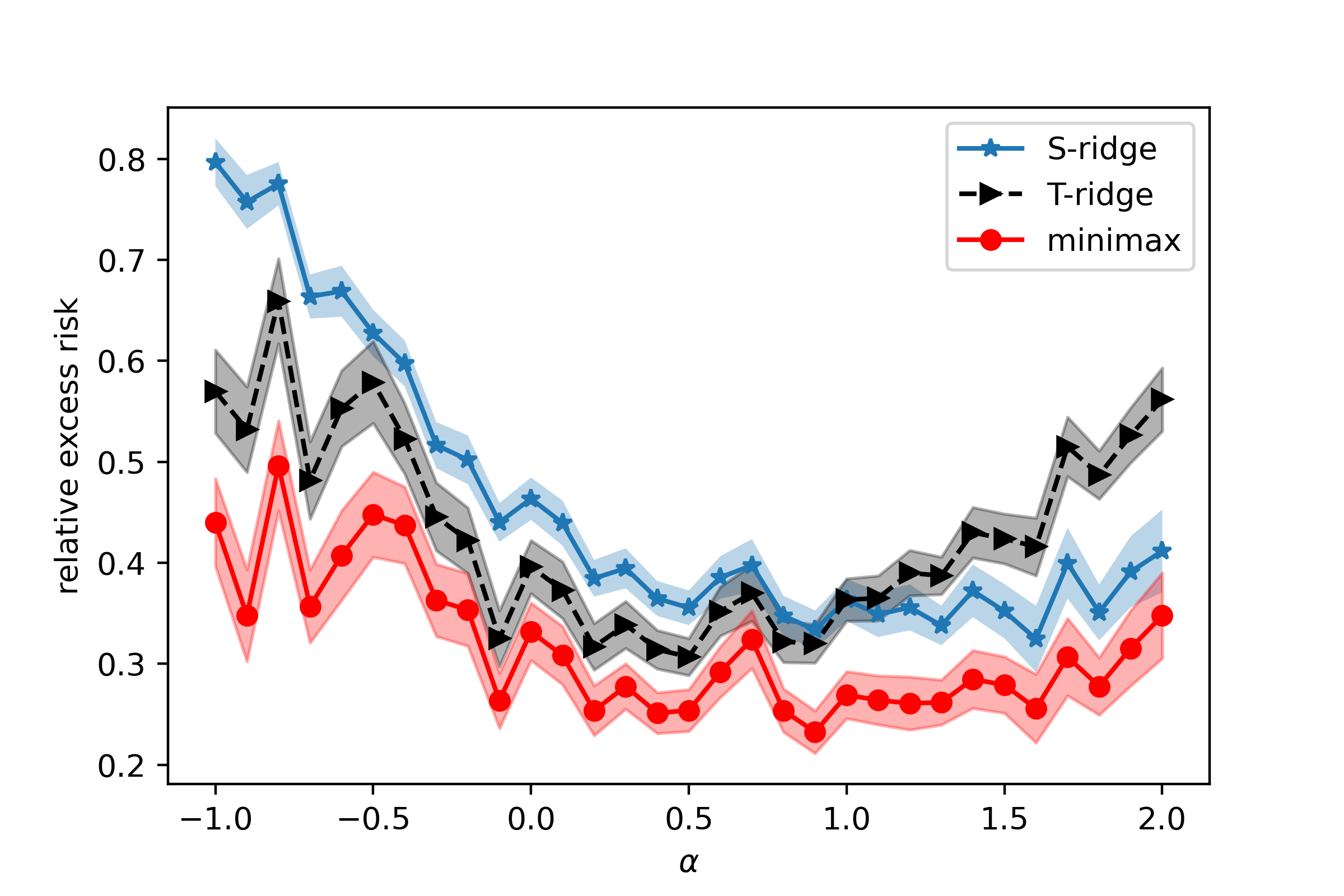}\hspace{-1.cm}	&
 \includegraphics[width=0.37\linewidth]{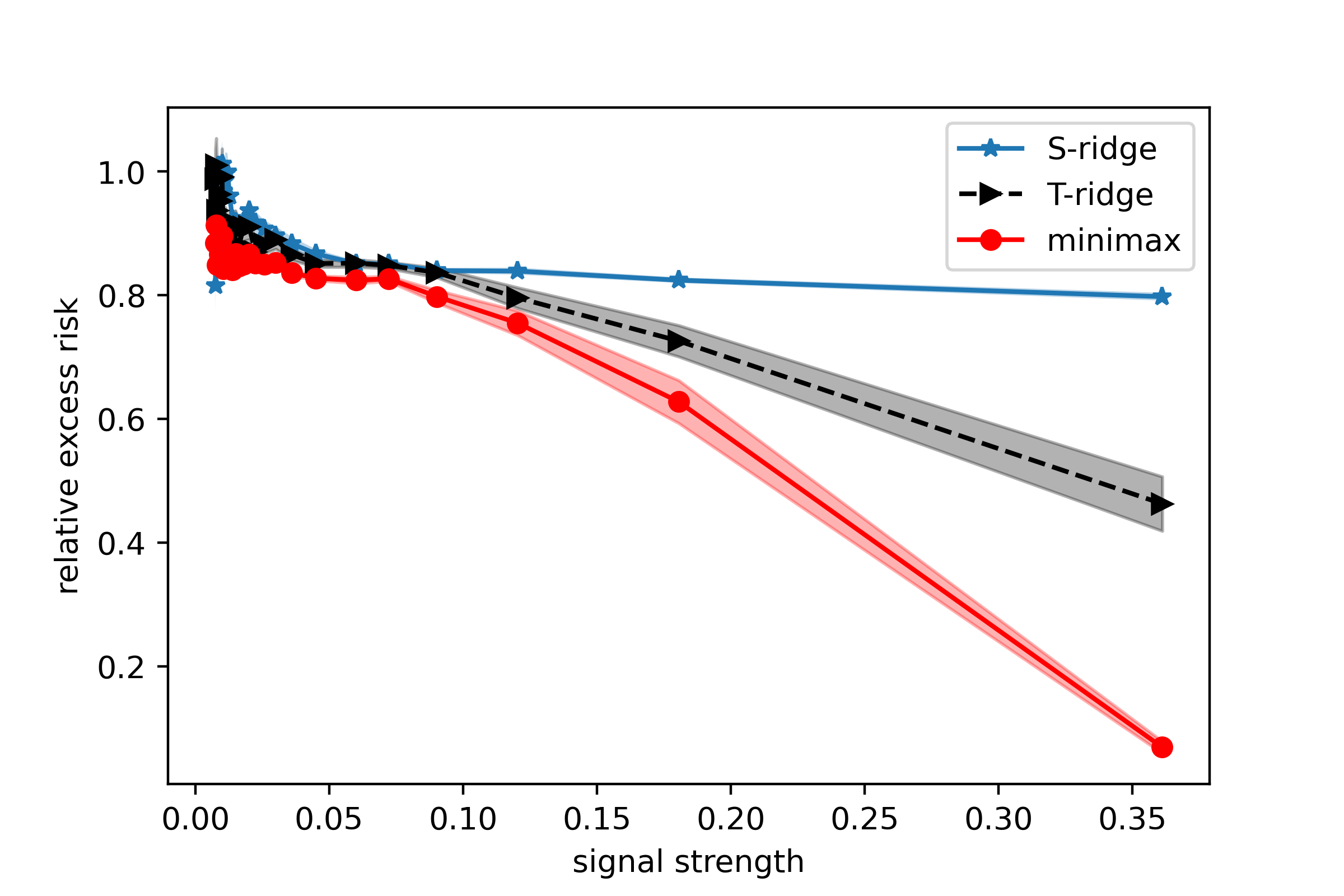}\hspace{-1.cm} &
\includegraphics[width=0.37\linewidth]{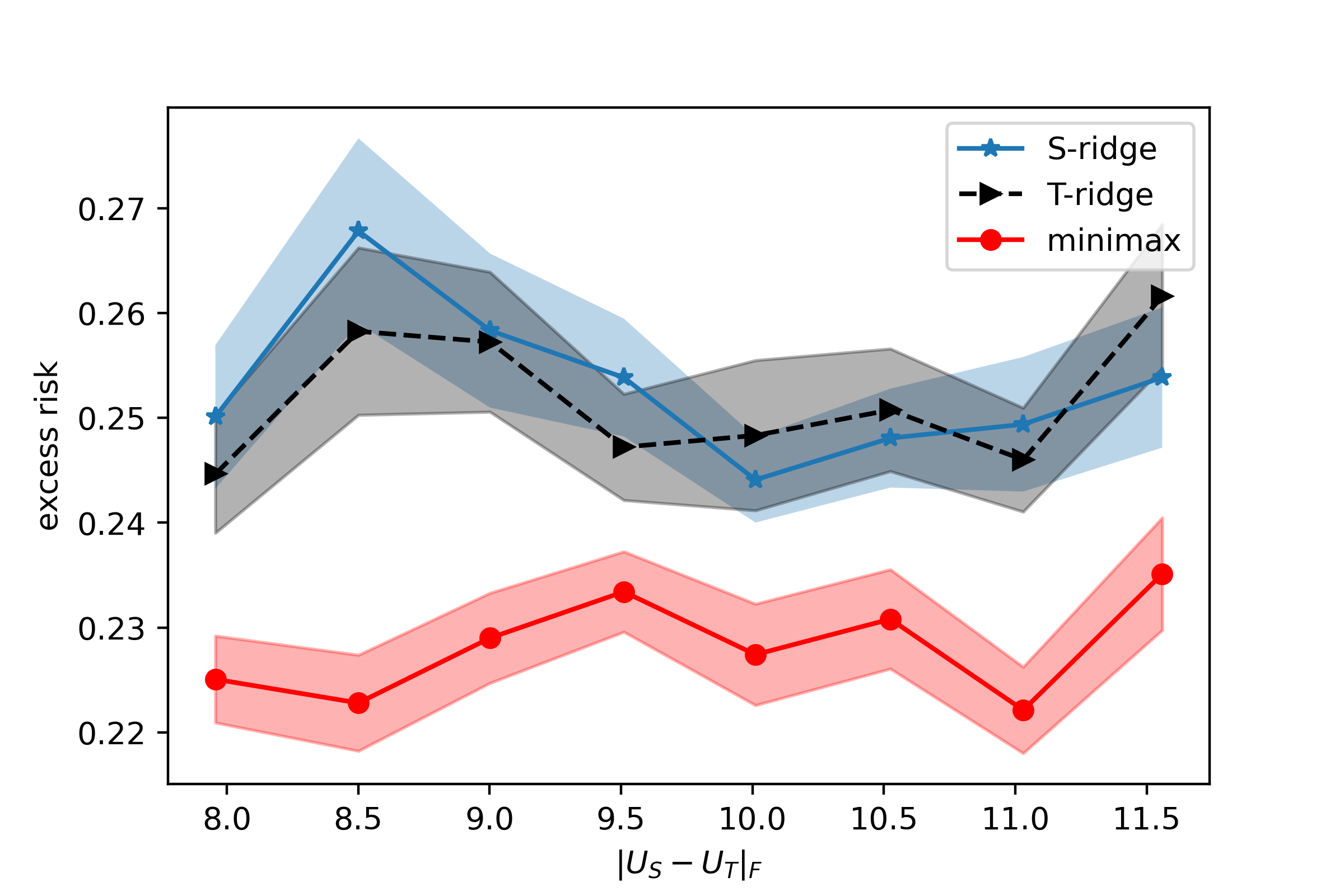}	\\
	(a)	covariate eigen-spectrum & (b) signal strength & (c) covariate eigenspace
\end{tabular}
	\caption{(a): The x-axis $\alpha$ defines the spread of eigen-spectrum of $\Sigma_S$:  $s_i \propto 1/i^{\alpha}, t_i\propto 1/i$. (b) x-axis is the normalized value of signal strength: $\|\Sigma_T\bbeta^*\|/r$. (c) X-axis is the covariate shift due to eigenspace shift measured by $\|U_S-U_T\|_F$. }
	\label{figure:appendix} 
\end{figure*} 
Figure \ref{figure:appendix} (a)(b) show similar phenomenon as the small noise setting presented in the main text.
From Figure \ref{figure:appendix} (c) we see no particular relationship between the performance of each algorithm with eigenspace shift.

\end{document}